%% file: resubmission.tex
\documentclass[11pt]{article}
\usepackage[preprint]{acl}

\usepackage{times,latexsym}
\usepackage[T1]{fontenc}
\usepackage[utf8]{inputenc}
\usepackage{microtype}
\usepackage{graphicx,booktabs,array}
\usepackage{amsmath,amssymb,amsthm}
\usepackage[table]{xcolor}
\definecolor{tableHead}{HTML}{EAF0F5}
\definecolor{tableBase}{HTML}{F2F3F5}
\definecolor{tableFocus}{HTML}{E7F2ED}
\usepackage{algorithm}
\usepackage{algpseudocode}
\usepackage{placeins}
\input{math_commands.tex}

\usepackage[capitalize,noabbrev]{cleveref}
\crefname{section}{Sec.}{Secs.}
\Crefname{section}{Sec.}{Secs.}
\crefname{table}{Tab.}{Tabs.}
\Crefname{table}{Tab.}{Tabs.}
\crefname{figure}{Fig.}{Figs.}
\Crefname{figure}{Fig.}{Figs.}
\crefname{algorithm}{Alg.}{Algs.}
\Crefname{algorithm}{Alg.}{Algs.}
\crefname{equation}{Eq.}{Eqs.}
\Crefname{equation}{Equation}{Equations}
\newcommand{\appref}[1]{Appendix~\labelcref{#1}}
\newtheorem{lemma}{Lemma}

\newtheorem{proposition}{Proposition}
\newtheorem{corollary}{Corollary}
\usepackage{array}
\newcolumntype{Y}{w{c}{24pt}}
\newcolumntype{Z}{w{c}{22pt}}
\newcolumntype{V}{w{c}{26pt}}
\newcolumntype{X}{w{c}{27pt}}
\newcolumntype{G}{w{c}{32pt}}
\newcommand{\pol}{\pi_{\vtheta}}
\newcommand{\refpol}{\pi_{\mathrm{ref}}}

\newcommand{\energy}{\mathcal E}
\newcommand{\dpo}{\mathrm{DPO}}
\newcommand{\metric}[2]{$#1{\scriptstyle\pm#2}$}
\newcommand{\bestmetric}[2]{\metric{#1}{#2}}

\title{Stochastic Pruning of the DPO Computation Graph:\\ 1.4$\times$ Faster Training and Lower Peak Memory}

\author{
  Yubing Wu\textsuperscript{1,\textdagger}\quad Wei Chen\textsuperscript{1,\textdagger}\quad Junmei Yang\textsuperscript{1}\quad Delu Zeng\textsuperscript{1,*}\quad Tiewen Pan\textsuperscript{2}
  \AND
  \textsuperscript{1}South China University of Technology\\
  \textsuperscript{2} \large\bfseries Shenzhen Technology University\\
  \textsuperscript{\textdagger}Equal contribution.\quad
  \textsuperscript{*}Corresponding author: \texttt{dlzeng@scut.edu.cn}
}

\begin{document}
\maketitle

\begin{abstract}
Direct preference optimization (DPO) incurs training costs in output scoring, sequence execution, and backpropagation, so reducing supervision alone does not necessarily save computation. We introduce a unified framework for stochastic pruning of the DPO computation graph that independently samples response positions for scoring, response prefixes for execution, and preference branches for backpropagation. By making these selections before the corresponding operations, the framework omits computation at each stage of a training step. Our analysis characterizes the objective changes induced by position and sequence pruning and uses inverse-probability weighting to obtain a conditionally unbiased single-branch gradient estimator for the sampled objective. Controlled measurements show that position and sequence pruning primarily reduce peak memory, with most output-memory savings attributable to an implementation that retains full response supervision. Branch pruning yields a $1.426\times$ step-time speedup over a concatenated DPO baseline and reduces training time by about 30\% relative to stock DPO across two model families, with slightly lower mean accuracy and benchmark-dependent changes in open-ended scores. These findings relate the resource benefits of stochastic graph pruning to the computation omitted and the resulting changes in the training objective.
\end{abstract}

\section{Introduction}
\label{sec:intro}
Direct preference optimization (DPO) formulates alignment as a logistic comparison of policy--reference likelihood ratios, with gradient contributions from both chosen and rejected responses \citep{rafailov2023direct}. In large-vocabulary language-model training, materialized logits can dominate memory consumption \citep{cce2025}, and attention I/O and operator fusion also affect memory use and training throughput \citep{dao2022flashattention,liger2025}. These costs motivate our question of which computations a preference-training step can omit.

\begin{figure*}[t]
\centering
\includegraphics[width=0.95\textwidth]{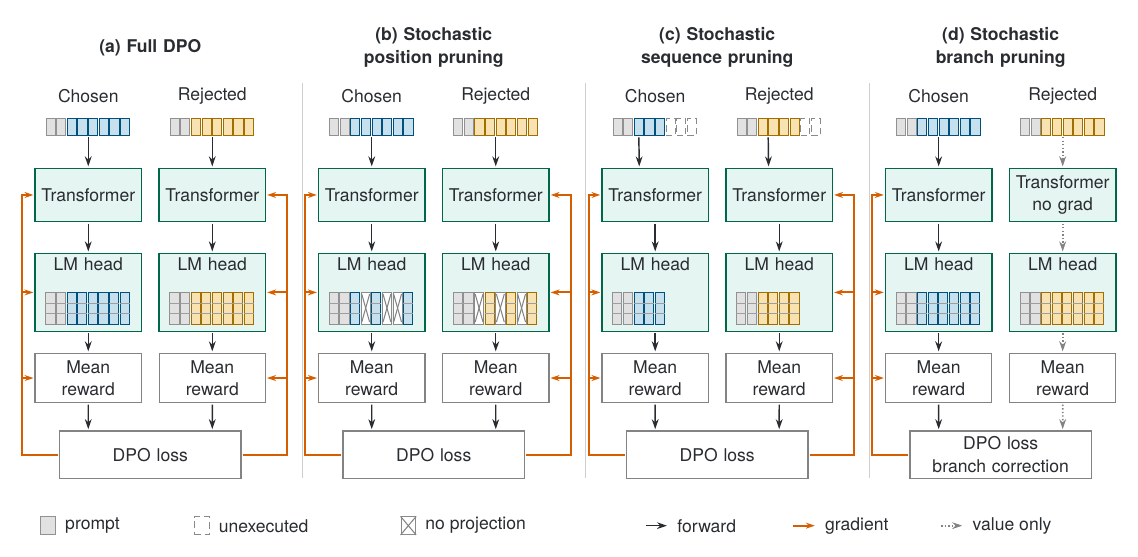}
\caption{Stochastic interventions on the DPO computation graph. The two response branches share policy parameters. Rewards average policy--reference log-ratios over scored response positions. Fixed-reference computations use matching prefixes and positions and are omitted. Sampling outcomes and counts are schematic.}
\label{fig:axes}
\end{figure*}

Prior work changes token supervision through confidence-based token selection or weighting \citep{confpo-2025,bmc2025}, reward-based selection \citep{sepo-2024}, and weighting from contrastive model predictions \citep{tis-dpo-2024}, while memory-efficient kernels change output-score computation \citep{cce2025,liger2025}. This distinction motivates our question: \emph{which parts of a DPO step can be omitted, and how do their resource savings relate to changes in the optimization objective?}

Preference-objective structure affects optimization \citep{chen2026disentangled}, while shared-parameter updates can couple response likelihoods \citep{yuan2025a} and reduce preferred-response probability \citep{razin2025unintentional}. Our token-level analysis examines this coupling (\appref{app:coupled}), and \cref{prop:random-loss} separates reward preservation in expectation from the nonlinear loss.

We introduce a framework that samples the DPO computation graph to control training cost (\cref{fig:axes}). Each training step jointly selects the executed response prefixes, the positions used to estimate rewards, and the branches that carry gradients. These selections precede sequence execution, vocabulary projection, and backpropagation, allowing computation to be omitted throughout the step. This view unifies position, sequence, and branch pruning through their effects on the sampled objective and its gradient. Our analysis characterizes the objective changes induced by reward sampling, while inverse-probability weighting makes the branch gradient conditionally unbiased for that objective.

Our contributions are:
\begin{itemize}
\item a unified framework for sampling DPO computation across response execution, token scoring, and backpropagation, with separately configurable sampling decisions.
\item a characterization of the position and prefix objective changes, tolerance conditions that relate their optimal gaps to the retention rate and the truncation level, and a conditionally unbiased branch estimator with a stated convergence rate.
\item controlled comparisons showing that full supervision retains most output-memory savings, that branch pruning reduces training time across model families, and that mean accuracy is slightly lower while open-ended score changes vary by benchmark (\cref{sec:efficiency,sec:quality,sec:c-validation}).
\end{itemize}

\section{Related Work}
\label{sec:related}
\paragraph{Token-level credit in preference optimization.}
A growing family of methods changes where DPO assigns credit through confidence-based token selection \citep{confpo-2025}, confidence-dependent weighting of varied tokens \citep{bmc2025}, trained oracles for token rewards \citep{sepo-2024}, and contrastive token weights \citep{tis-dpo-2024}. Other approaches use optimal-transport reweighting \citep{li2025optimal}, gradient-attribution priors \citep{yang2026tokenimportance}, prefix-level ratio matching \citep{tbpo-2025}, implicit-reward credit \citep{sedpo2026}, attention-derived weights \citep{attentionpo2026}, minimal-edit emphasis in clinical dialogue \citep{tddpo2026}, or token-level KL control \citep{zeng2024token}. Our framework places supervision selection within the computation of a complete DPO step. Matched selection controls motivate uniform membership (\cref{sec:motivation}), and full-retention controls distinguish supervision deletion from implementation savings (\cref{tab:a-attribution}).

\paragraph{Efficient preference training.}
Cut Cross-Entropy avoids materializing the token-by-vocabulary logit matrix \citep{cce2025}, and Liger-Kernel uses fusion and chunked cross-entropy computation to reduce memory and improve training throughput \citep{liger2025}. FlashAttention relates attention efficiency to memory accesses rather than FLOP counts alone \citep{dao2022flashattention}, while activation checkpointing trades recomputation for lower activation storage \citep{chen2016trainingdeepnetssublinear}. Related efficiency work covers flow training \citep{du2022toflow} and diffusion inference \citep{li2025evodiff}. SePO reports reduced GPU training hours through offline token selection \citep{sepo-2024}. Our framework considers how sampling throughout a DPO step changes both executed computation and the training objective. Controlled profiles separate implementation savings from the effects of reduced supervision and execution (\cref{sec:efficiency}).

\paragraph{Stochastic computation and optimization dynamics.}
Dropout randomly removes units \citep{srivastava2014dropout} and stochastic depth bypasses sampled layers \citep{huang2016deep}, while importance sampling can reduce stochastic-gradient variance using gradient-dependent probabilities \citep{katharopoulos2018importance}. Preference-optimization analyses study gradient entanglement \citep{yuan2025a}, likelihood displacement \citep{razin2025unintentional}, logit-space squeezing \citep{logitssam2026}, and gated rejected contributions \citep{gate-dpo-2025}. Density-ratio estimation provides a distribution-comparison perspective \citep{chen2025dequantified} for preference objectives such as DDRO \citep{higuchi2025ddro}, whose update structure is analyzed alongside margin-based losses \citep{chen2026disentangled}. Our analysis separates objective changes from gradient-estimation noise under composed sampling (\cref{prop:branch-unbiased,sec:composition}). Shared-parameter response changes are analyzed in \appref{app:coupled}.

\section{Stochastic Graph Pruning}
\label{sec:method}
\subsection{Formulation}
\label{sec:problem}
\label{sec:objective}
\label{sec:notation}
\label{sec:graph}
Let $\train=\{(\vx,\vy_w,\vy_l)\}$ contain prompts $\vx$ with chosen and rejected responses $\vy_w,\vy_l$. Write $\pol$ for the policy with trainable parameters $\vtheta$ and $\refpol$ for the fixed reference \citep{rafailov2023direct}. For side $b\in\{w,l\}$, let $T_b\geq1$ count response tokens, excluding prompt and padding. At position $t$, define the target-token log-ratio
\begin{equation}
u_{b,t}=\log\frac{\pol(\evy_{b,t}\mid\vx,\vy_{b,<t})}{\refpol(\evy_{b,t}\mid\vx,\vy_{b,<t})}.
\label{eq:log-ratio}
\end{equation}
Here $\evy_{b,t}$ is the target token and $\vy_{b,<t}$ its preceding response prefix. With mean rewards $R_b=T_b^{-1}\sum_{t=1}^{T_b}u_{b,t}$, margin $\Delta=R_w-R_l$, and scale $\beta>0$, DPO minimizes the loss \citep{rafailov2023direct}:
\begin{equation}
F(\vtheta)=\E_{\train}[\Ls_{\dpo}]=\E_{\train}[\log\bigl(1+e^{-\beta \Delta}\bigr)].
\label{eq:dpo-loss}
\end{equation}
This length normalization parallels SimPO \citep{meng2024simpo}, but retains policy--reference log-ratios.

\Cref{fig:axes} separates three costs: dense vocabulary projection materializes $T_b\times V$ logits for vocabulary size $V$ \citep{cce2025}, sequence execution computes attention and stores or recomputes activations \citep{dao2022flashattention,chen2016trainingdeepnetssublinear}, and backpropagation includes both responses \citep{rafailov2023direct}.

Write $\sI_M=\{1,\ldots,M\}$. On side $b$, sequence pruning selects length $M_b\in\sI_{T_b}$ through cutoff fraction $\rho_b\in(0,1]$, retaining $\tilde\vy_b=\vy_{b,1:M_b}$. Position pruning selects $\sS_b\subseteq\sI_{M_b}$ with $k_b=|\sS_b|$ controlled by retention $\alpha\in(0,1]$. Branch pruning selects $b^\star\in\{w,l\}$ with probabilities $q_w,q_l>0$, $q_w+q_l=1$. All sampling laws are parameter-independent, and rewards average over selected positions. We use $\sigmoid$ for the sigmoid and $\sg$ for stop-gradient, which preserves values and blocks derivatives. \appref{app:notation} summarizes the notation.

\subsection{Stochastic Position Pruning}
\label{sec:method-position}
Position pruning retains a uniform subset of response positions for scoring while preserving the executed response as context. Draw $\sS_b\sim\positionsample(M_b,\alpha)$ by
\begin{equation}
\begin{aligned}
k_b&=\max(1,\ceil{\alpha M_b}),\\
\positionsample(M_b,\alpha)
&=\operatorname{Unif}\{\sS\subseteq\sI_{M_b}:|\sS|=k_b\}.
\end{aligned}
\label{eq:position-sampling}
\end{equation}
Standalone position pruning uses $M_b=T_b$. Define the selected-position reward
\begin{equation}
R_b(\sS_b)=\frac1{k_b}\sum_{t\in\sS_b}u_{b,t},
\label{eq:masked-reward}
\end{equation}
with masked margin $\Delta=R_w(\sS_w)-R_l(\sS_l)$.
The subset is drawn before the language-model head, without computing policy probabilities. Applying the head only at $\sS_b$ evaluates \cref{eq:masked-reward} exactly, skipping $M_b-k_b$ vocabulary projections, their logits, and their backward traffic.

Uniform sampling preserves the mean reward, but the nonlinear loss introduces a gap.

\begin{proposition}
\label{prop:random-loss}
For a fully executed response, let $\sU_k$ be a uniform size-$k$ subset of $\sI_T$ and $\vphi_t=\nabla_{\vtheta}\log\pol(\evy_t\mid\vx,\vy_{<t})$. Let $\Delta_U$ be the margin under independent side masks and $\Delta_{\mathrm{full}}$ the full-response margin. At any fixed policy,
\begin{align}
\E[R(\sU_k)]&=R(\sI_T),\nonumber\\
\E\Bigl[\frac1k\sum_{t\in\sU_k}\vphi_t\Bigr]&=\frac1T\sum_{t=1}^T\vphi_t,
\label{eq:random-unbiased}
\end{align}
and
\begin{equation}
0\leq\E[\ell(\Delta_U)]-\ell(\Delta_{\mathrm{full}})\leq\frac{\beta^2}{8}\Var(\Delta_U),
\label{eq:random-loss-gap}
\end{equation}
where $\Var$ is over the sampled masks and $\Var(\Delta_U)=\sum_{b:T_b>1}\frac{T_b-k_b}{k_b(T_b-1)}\,v_b$. Here $\bar u_b=T_b^{-1}\sum_tu_{b,t}$ and $v_b=T_b^{-1}\sum_t(u_{b,t}-\bar u_b)^2$ are the population mean and variance. Single-token sides contribute zero variance.
\end{proposition}

At $\alpha=1$, selected-head execution preserves the full-response objective while avoiding projection outside valid response positions (\cref{tab:a-attribution}).

To bound the sampled optimum, assume throughout this subsection and \cref{app:proof-optimality,app:proof-bias} that $|u_{b,t}|\leq U$, $\|\nabla_{\vtheta}u_{b,t}\|\leq G$, $v_b\leq\bar v$, and $T_b\geq T_0\geq2$ on the analyzed region. The objectives are twice differentiable there and attain their minima.

\begin{corollary}[Optimality gap of position pruning]
\label{cor:position-optimality}
Let $\vtheta^*$ minimize the full-response objective $F$ and $\vtheta_A^*$ minimize $F_A(\vtheta)=\E_{\train}\E_S[\ell(\Delta_U)]$. Then
\begin{equation}
0\leq F(\vtheta_A^*)-F(\vtheta^*)\leq\frac{\beta^2}{8}\E_{\train}\bigl[\Var_S(\Delta_U)\bigr].
\label{eq:position-optimality}
\end{equation}
\end{corollary}

The variance is evaluated at $\vtheta^*$ over masks, then averaged over data. Since $k_b\geq\alpha T_b$, \cref{prop:random-loss} gives $\E_{\train}[\Var_S(\Delta_U)]\leq2\bar v\,r_A$, where
\begin{equation}
r_A:=\frac{1-\alpha}{\alpha(T_0-1)}.
\label{eq:rA}
\end{equation}
Thus the bound decays as $1/T_0$ at fixed retention and vanishes at $\alpha=1$. A sufficient condition for loss tolerance $\varepsilon_F>0$ is
\begin{equation}
\alpha\geq\frac{1}{1+4\varepsilon_F(T_0-1)/(\beta^2\bar v)},
\label{eq:alpha-rule}
\end{equation}

A second-order expansion also gives $\|\nabla F_A-\nabla F\|=\mathcal{O}(r_A)$ (\appref{app:proof-bias}). If $\nabla^2F\succeq\mu I$, $\mu>0$, along the segment joining interior stationary points of $F_A$ and $F$ within the analyzed region, their displacement is $\mathcal{O}(r_A)$ rather than $\mathcal{O}(\sqrt{r_A})$. The loss bound requires no curvature assumption. Sufficient curvature conditions appear in \appref{app:proof-curvature}.

\subsection{Stochastic Sequence Pruning}
\label{sec:method-prefix}
Sequence pruning truncates each response before policy and reference execution. For minimum retained fraction $\rho_{\min}\in(0,1)$, draw $(\tilde\vy_b,M_b)\sim\prefixsample(\vy_b,\rho_{\min})$ by
\begin{equation}
\begin{aligned}
\rho_b&\sim\mathcal U(\rho_{\min},1),\\
M_b&=\max(1,\ceil{\rho_bT_b}),\\
\tilde\vy_b&=\vy_{b,1:M_b}.
\end{aligned}
\label{eq:prefix-sampling}
\end{equation}

Both models repack the prompt with $\tilde\vy_b$, removing suffix attention, activations, and recomputation. The reward is $R_b=M_b^{-1}\sum_{t=1}^{M_b}u_{b,t}$, and $\rho_b\equiv1$ recovers the full-response objective. \Cref{sec:sensitivity} evaluates $\rho_{\min}$.

Prefix inclusion and length normalization favor earlier positions:
\begin{lemma}[Implicit temporal weighting]
\label{lem:temporal-weight}
For one response of length $T$, let $M\in\{1,\ldots,T\}$ be the sampled cutoff, with law $q(m)=\Pr(M=m)$ and prefix reward $R_M=M^{-1}\sum_{t=1}^{M}u_t$. Define the temporal weights $w_t:=\sum_{m=t}^{T}q(m)/m$. Then
\begin{equation}
\E_M[R_M]=\sum_{t=1}^{T}w_tu_t.
\label{eq:temporal-weight}
\end{equation}
The weights satisfy $w_1\geq\cdots\geq w_T\geq0$ and $\sum_{t=1}^{T}w_t=1$. Uniform weights $w_t=1/T$ occur exactly when $M=T$ almost surely.
\end{lemma}
\appref{app:proof-temporal} proves this identity and characterizes the cutoff laws. The nonlinear loss adds a Jensen gap:

\begin{lemma}[Prefix Jensen gap]
\label{lem:prefix-gap}
Let $\Delta_M$ be the random-prefix margin, $\mu_{\mathrm{pref}}=\E[\Delta_M]$ its mean, and $\Ls_{\mathrm{pref}}=\E[\ell(\Delta_M)]$ the expected prefix loss with Jensen gap $J_{\mathrm{pref}}$. Then
\begin{equation}
\begin{aligned}
\Ls_{\mathrm{pref}}&=\ell(\mu_{\mathrm{pref}})+J_{\mathrm{pref}},\\
0\leq J_{\mathrm{pref}}&\leq\frac{\beta^2}{8}\Var(\Delta_M).
\end{aligned}
\label{eq:prefix-gap}
\end{equation}
\end{lemma}

Unlike position pruning, sequence pruning can shift the mean margin. The shift $\ell(\mu_{\mathrm{pref}})-\ell(\Delta_{\mathrm{full}})$ has no fixed sign, whereas $J_{\mathrm{pref}}\geq0$ vanishes for deterministic cutoffs or $\beta=0$ (\appref{app:proof-prefix}). Parameter-independent cutoffs make $\nabla_{\vtheta}\ell(\Delta_M)$ unbiased for $\nabla_{\vtheta}\Ls_{\mathrm{pref}}$, generally different from the full-response gradient. \Cref{sec:quality} evaluates the resulting performance.

Assume causal execution with the prompt retained, so prefix token scores and gradients match full execution. For truncation fraction $h_b=1-M_b/T_b$, \cref{eq:prefix-sampling} gives $\E[h_b]\leq(1-\rho_{\min})/2$. Prefix and full mean rewards differ by at most $2Uh_b$, and their gradients by $2Gh_b$. The $\beta$-Lipschitz loss therefore yields
\begin{equation}
|F_P(\vtheta)-F(\vtheta)|\leq2\beta U(1-\rho_{\min}),
\label{eq:prefix-bias}
\end{equation}
where $F_P=\E[\Ls_{\mathrm{pref}}]$. Consequently, $F(\vtheta_P^*)-F(\vtheta^*)\leq4\beta U(1-\rho_{\min})$. Unlike \cref{eq:position-optimality}, this bound has no response-length decay. At the tested $\rho_{\min}=1/2$, the optimality bound equals the full loss range over $|\Delta|\leq2U$, making it vacuous. This comparison concerns theoretical bounds, not measured deviations. \appref{app:proof-prefix-bias} also derives gradient and displacement bounds.

\subsection{Stochastic Branch Pruning}
\label{sec:method-branch}
Branch pruning removes one backward graph while preserving the expected gradient conditional on the sampled positions and prefixes. With $\lambda=\beta\sigmoid(-\beta\Delta)$ and $\sigmoid(z)=(1+e^{-z})^{-1}$, the DPO gradient decomposes as \citep{rafailov2023direct}
\begin{equation}
\nabla_{\vtheta}\Ls_{\dpo}
=-\lambda\,\nabla_{\vtheta}R_w+\lambda\,\nabla_{\vtheta}R_l.
\label{eq:branch-decomposition}
\end{equation}
Independently of positions and prefixes, draw the active backward-branch set and its probability, $(\sB,q^\star)\sim\branchsample(q_w)$, with $q_l=1-q_w$ and
\begin{equation}
\Pr\bigl[(\sB,q^\star)=(\{b\},q_b)\bigr]=q_b,
\; b\in\{w,l\}.
\label{eq:branch-sampling}
\end{equation}
Only $b^\star\in\sB$ retains a policy backward graph. The opposite side runs under \texttt{no\_grad}, providing the other reward needed for $\lambda$. Disabling branch pruning sets $\sB=\{w,l\}$ and $q^\star=1$. Inverse-probability weighting corrects each selected contribution. With $\softplus(z)=\log(1+e^z)$, use
\begin{equation}
\begin{aligned}
\widehat\Ls_w&=\frac1{q_w}\,\softplus\bigl(-\beta\bigl(R_w-\sg(R_l)\bigr)\bigr),\\
\widehat\Ls_l&=\frac1{q_l}\,\softplus\bigl(-\beta\bigl(\sg(R_w)-R_l\bigr)\bigr).
\end{aligned}
\label{eq:branch-surrogate}
\end{equation}
The sampled loss $\widehat\Ls=\widehat\Ls_{b^\star}$ preserves the numerical margin and differentiates only the selected reward.

\begin{proposition}[Conditional unbiasedness]
\label{prop:branch-unbiased}
Fix the sampled response positions and prefixes. Then
$\E_{b^\star}\bigl[\nabla_{\vtheta}\widehat\Ls\bigr]=\nabla_{\vtheta}\log\bigl(1+e^{-\beta(R_w-R_l)}\bigr)$.
With position and sequence pruning disabled, this is the full-response per-pair gradient in \cref{eq:dpo-loss}.
\end{proposition}

We default to $q_w=q_l=0.5$ and sweep $q_w\in\{0.3,0.5,0.7\}$ in \cref{sec:sensitivity}. Gradient norms govern importance-sampling variance \citep{katharopoulos2018importance}, as specialized in \appref{app:proof-branch}.

\begin{proposition}[Branch estimator variance]
\label{prop:branch-variance}
Fix positions and prefixes, and set $\vc_w=-\lambda\nabla_{\vtheta}R_w$, $\vc_l=\lambda\nabla_{\vtheta}R_l$. The estimator $\hat\vg=\vc_{b^\star}/q_{b^\star}$ has mean $\vc_w+\vc_l$ and variance
\begin{equation}
\Var(\hat\vg)=\frac{\|\vc_w\|^2}{q_w}+\frac{\|\vc_l\|^2}{q_l}-\|\vc_w+\vc_l\|^2 .
\label{eq:branch-variance}
\end{equation}
\end{proposition}

At $q_w=q_l=1/2$, the variance is $\|\vc_w-\vc_l\|^2\leq4\beta^2G^2$ under the bounds of \cref{sec:method-position}. Averaging $m$ estimators independent in preference pairs, positions, prefixes, and branches divides the variance by $m$. Re-drawing only the branch reduces only its variance component.

Branch pruning adds gradient noise without objective bias. For the optimized objective $\Phi$, assume $\mu$-strong convexity and $L$-smoothness on a closed convex region $\Theta$ with minimizer $\vtheta_\Phi^*$ in its interior. Starting from $\vtheta_0\in\Theta$, projected SGD with $m$ independent estimators per step and the schedule in \appref{app:proof-branch-rate} gives $\E\|\vtheta_N-\vtheta_\Phi^*\|^2=\mathcal{O}(1/N)$. That appendix supplies the constants and a high-probability bound. Gradient unbiasedness does not imply unbiased finite-step parameters.

\subsection{Composition}
\label{sec:composition}
\Cref{alg:stochastic-dpo} composes the interventions in execution order: truncate the response, select scoring positions, then select the backward branch and apply its probability correction. Each mode is optional. Full disables all three, using full responses, all response scores, and both backward branches. Tables abbreviate the modes as Position, Sequence, and Branch, with ``+'' denoting composition.
\begin{algorithm}[ht]
\caption{Stochastic DPO step}
\label{alg:stochastic-dpo}
\small
\begin{algorithmic}[1]
\algtext*{EndIf}
\algtext*{EndFor}
\Require Pair $(\vx,\vy_w,\vy_l)$, policy $\pol$, frozen $\refpol$
\Require Rates $\alpha,\rho_{\min},q_w$, enabled pruning modes
\For{$b\in\{w,l\}$}
  \State $(\tilde\vy_b,M_b)\gets(\vy_b,T_b)$
  \If{sequence pruning}
    \State $(\tilde\vy_b,M_b)\sim\prefixsample(\vy_b,\rho_{\min})$ \Comment{\cref{eq:prefix-sampling}}
  \EndIf
  \State $\sS_b\gets\sI_{M_b}$
  \If{position pruning}
    \State $\sS_b\sim\positionsample(M_b,\alpha)$ \Comment{\cref{eq:position-sampling}}
  \EndIf
\EndFor
\State $\sB\gets\{w,l\}$
\State $q^\star\gets1$
\If{branch pruning}
  \State $(\sB,q^\star)\sim\branchsample(q_w)$ \Comment{\cref{eq:branch-sampling}}
\EndIf
\For{$b\in\{w,l\}$}
  \State $\vz_b\gets\logprobs(\pol,\vx,\tilde\vy_b,\sS_b,b\in\sB)$
  \State $\vz_b^{\mathrm{ref}}\gets\logprobs(\refpol,\vx,\tilde\vy_b,\sS_b,\texttt{false})$
  \State $R_b\gets\mean(\vz_b-\vz_b^{\mathrm{ref}})$ \Comment{\cref{eq:masked-reward}}
\EndFor
\State $\Ls\gets\ell(R_w-R_l)/q^\star$ \Comment{\cref{eq:branch-surrogate}}
\State $\Ls.\operatorname{backward}()$
\State Update $\vtheta$
\end{algorithmic}
\end{algorithm}

For a policy or reference model $\pi$, the scored log-probability vector is
\begin{equation}
\begin{aligned}
\logprobs(\pi,\vx,\vy,\sS,g)
&=\\[-2pt]
&\bigl(\log\pi(\evy_t\mid\vx,\vy_{<t})\bigr)_{t\in\sS}.
\end{aligned}
\label{eq:scored-logprobs}
\end{equation}
Both models score the same prefixes and positions, with entries in response order. Position pruning applies the head only at $\sS$; otherwise, the full head is evaluated before gathering. The argument $g$ controls gradient recording through the backbone and head: only policy branches in $\sB$ retain a backward graph, while the reference always runs without gradients. Numerical values are unchanged. The outputs $\vz_b,\vz_b^{\mathrm{ref}}$ give $R_b$ by the mean difference in \cref{eq:masked-reward}. The step loss is $\Ls=\ell(R_w-R_l)/q^\star$ (\cref{eq:branch-surrogate}), recovering the per-pair loss in \cref{eq:dpo-loss} on the selected prefixes and positions when $q^\star=1$.

The minimizer $\vtheta_{AP}^*$ of the composed objective $F_{AP}$ satisfies
\begin{equation}
\begin{aligned}
F(\vtheta_{AP}^*)-F(\vtheta^*)&\leq\frac{\beta^2\bar v(1-\alpha)}{4\alpha(M_0-1)}\\
&\quad+4\beta U(1-\rho_{\min}),
\end{aligned}
\label{eq:composed-loss}
\end{equation}
where $M_0\!\!=\!\!\max(2,\ceil{\rho_{\min}T_0})$.
The position-pruning bias bound decays as $1/T_0$ at fixed retention. The prefix bound scales with $1-\rho_{\min}$ without length decay, while branch pruning adds no objective bias and has the $\mathcal{O}(1/N)$ mean-squared optimization error above. \appref{app:proof-composed} proves the composed loss bound and \cref{eq:composed-bias,eq:error-budget}, which allocate a total tolerance across $\alpha$, $\rho_{\min}$, and $N$.

\FloatBarrier
\section{Experiments}
\label{sec:experiments}
\subsection{Experimental Setup}
\label{sec:setup}
We use Qwen3-4B, Qwen3-8B \citep{qwen3}, and Mistral-7B-Instruct-v0.3 \citep{jiang2023mistral7b}. The core resource profiles, eight-configuration factorial, and parameter sweeps use Qwen3-4B on HelpSteer3-Preference \citep{helpsteer3preference}. LoRA \citep{hu2021loralowrankadaptationlarge} has rank 64, scaling 128, and dropout $0.05$. Training uses FlashAttention-2 \citep{dao2023flashattention2fasterattentionbetter}, gradient checkpointing \citep{chen2016trainingdeepnetssublinear}, AdamW \citep{loshchilov2019decoupledweightdecayregularization}, DPO $\beta=1$. Global batch is 64 with per-device batch 2, accumulation 4, and eight Ascend 910B accelerators with 32\,GiB each. Optimizer and corpus details are in \appref{app:hs-recipe}; matched token-selection controls are in \cref{tab:hs-full,tab:hs-pairs}.
Defaults are $\alpha=0.6$, $\rho_b\sim\mathcal U(0.5,1)$, and $q_w=q_l=0.5$. The factorial and sweeps use training and selection seed 42. A second Qwen seed (43) and Mistral seed 42 test the Full--Branch comparison.

\paragraph{Metrics.}
We group benchmarks by evaluation protocol. Accuracy-based benchmarks comprise ARC-Challenge \citep{clark2018think}, GSM8K \citep{cobbe2021training}, TruthfulQA \citep{lin2022truthfulqa}, BBH \citep{suzgun2022challenging}, MATH-Hard \citep{hendrycks2021math}, MMLU-Pro \citep{wang2024mmlupro}, and MuSR \citep{sprague2024mus}. For open-ended generation, MiniMax-M3 \citep{minimax-m3} judges MT-Bench (MT, $0$--$10$) \citep{zheng2023judgingllmasajudgemtbenchchatbot}, Arena-Hard (AH, tie-adjusted win rate in \%) \citep{li2024crowdsourceddatahighqualitybenchmarks}, and WildBench (WB, $1$--$10$) \citep{lin2024wildbenchbenchmarkingllmschallenging}. Mean Accuracy averages the seven accuracy-based benchmark percentages. Mean Generation Score equally averages $10\,\mathrm{MT}$, $\mathrm{AH}$, and $10\,\mathrm{WB}$ on a 100-point scale. Task-level means are retained in \cref{tab:factorial-tasks,tab:sweep-full,tab:supplement-tasks}. 

\subsection{Resource Effects and Attribution}
\label{sec:efficiency}
\begin{figure*}[t]
\centering
\includegraphics[width=0.90\textwidth]{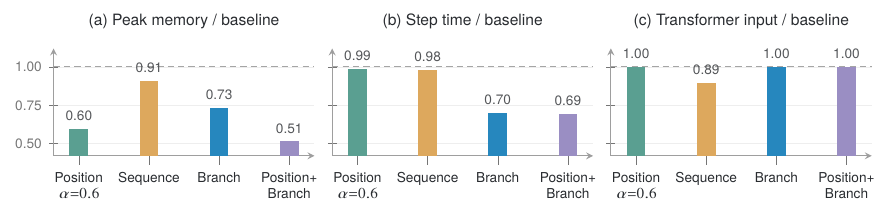}
\caption{Resource ratios at the default settings. Dashed lines mark each baseline. Panel (c) counts executed input tokens, including prompt and padding.}
\label{fig:profile}
\end{figure*}

\paragraph{Baselines and measurement.}
Full denotes the unpruned DPO baseline using the original padded stack. Full-concat batches chosen and rejected responses together to reduce padding; both retain all response supervision and both backward graphs. Position, Branch, and Position+Branch profiles use Full-concat; Sequence and end-to-end trainings use Full. The concatenated profiler measures Full-concat at $17.343$\,s and $23{,}627$\,MB, vs. Full at $18.204$\,s and $18{,}571$\,MB: $1.049\times$ faster with $27\%$ more memory. The separate prefix profiler measures Full at $18.826$\,s. Timings are compared within a profiler. The concatenated profile uses 16 paired measured rounds. Peak high-bandwidth memory (HBM) is reset per step and maximized over 8 ranks.

\paragraph{Output head versus supervision.}
Position combines selected-head execution with random position pruning. In the separate attribution profile (\cref{tab:a-attribution}), retaining all supervision already saves $8{,}905$\,MB; sampling adds only $642$\,MB. Thus selected-head execution accounts for $93.3\%$ of the saving, vs. $6.7\%$ from supervision deletion. Full-retention rewards match Full-concat exactly, and selected-head and dense-reduced rewards match under identical masks (\appref{app:supplement}). The implementation supplies most memory savings, while deletion introduces the loss gap of \cref{prop:random-loss}.

\begin{table}[ht]
\centering
\small
\setlength{\tabcolsep}{4pt}
\begin{tabular}{lccc}
\toprule
\rowcolor{tableHead}
Configuration & Retention & Step (s) & HBM (MB)\\
\midrule
\rowcolor{tableBase}
Full-concat & 1.0 & 17.081 & 23{,}627\\
\rowcolor{tableFocus}
Selected head & 1.0 & 16.827 & 14{,}722\\
Selected head & 0.6 & 16.791 & 14{,}080\\
\bottomrule
\end{tabular}
\caption{Qwen output-head attribution in a separate matched profile. Retention $1.0$ preserves all response supervision. Shading highlights full-supervision selected-head execution. Details: \appref{app:supplement}.}
\label{tab:a-attribution}
\end{table}

\vspace{-3mm}
\paragraph{Backward computation.}
Against Full-concat, Branch reduces step time by $29.9\%$ ($1.426\times$ speedup) and peak HBM by $26.7\%$. Position+Branch reaches $1.443\times$ speedup and $48.6\%$ lower HBM (\cref{fig:profile}). Its $11{,}471$\,MB saving is below the $15{,}855$\,MB sum of the separate savings; the cause of this sub-additivity remains unresolved.

\paragraph{Sequence execution.}
Sequence pruning executes $0.85$--$0.94$ of Full's input tokens and saves $8$--$11\%$ peak HBM (\cref{tab:b-profile}). Across random and fixed cutoffs, steps take $18.430$--$18.482$\,s, at most $2.1\%$ below Full. Differences between cutoffs are smaller than the $0.18$--$0.44$\,s spread across repeated measurements of one configuration. Token reduction therefore has a much smaller timing effect than branch pruning at this geometry; its physical cause remains unresolved.

\begin{table}[ht]
\centering
\small
\setlength{\tabcolsep}{2.8pt}
\begin{tabular}{lrrrr}
\toprule
\rowcolor{tableHead}
Sequence setting & Tokens & HBM (MB) & $\Delta$HBM & Step (s)\\
\midrule
\rowcolor{tableBase}
Full & 1.000$\times$ & 18{,}571 & --- & 18.826\\
$\mathcal U(0.25,1)$ & 0.854$\times$ & 16{,}718 & $-10.0\%$ & 18.480\\
\rowcolor{tableFocus}
$\mathcal U(0.5,1)$ & 0.893$\times$ & 16{,}827 & $-9.4\%$ & 18.430\\
$\mathcal U(0.75,1)$ & 0.943$\times$ & 17{,}072 & $-8.1\%$ & 18.482\\
fixed $0.75$ & 0.894$\times$ & 16{,}561 & $-10.8\%$ & 18.471\\
\bottomrule
\end{tabular}
\caption{Sequence profile against Full. Tokens counts executed input, including prompt and padding, relative to Full. Shading marks the default $\mathcal U(0.5,1)$.}
\label{tab:b-profile}
\end{table}

\subsection{Parameter Sensitivity}
\label{sec:sensitivity}

\Cref{tab:param-sweep} varies one axis at a time using the factorial recipe. Realized pruning fractions are reported in \appref{app:hs-factorial}. Sequence input counts include prompt and padding, so fractions are not comparable across axes or directly to $\rho_b$.
\begin{table}[!ht]
\centering
\small
\setlength{\tabcolsep}{1.8pt}
\begin{tabular}{llrrr}
\toprule
\rowcolor{tableHead}
Axis & Setting & Resource & {\scriptsize\shortstack{Mean\\accuracy (\%)}} & {\scriptsize\shortstack{Mean generation\\score}}\\
\midrule
Position & $\alpha{=}0.30$ & 11{,}211 MB & 57.90 & 47.73\\
\rowcolor{tableFocus}
Position & $\alpha{=}0.60$ & 11{,}320 MB & 58.06 & 48.38\\
Position & $\alpha{=}0.90$ & 11{,}580 MB & 58.32 & 48.63\\
\midrule
Sequence & $\mathcal U(0.25,1)$ & 16{,}718 MB & 58.09 & 47.91\\
\rowcolor{tableFocus}
Sequence & $\mathcal U(0.50,1)$ & 16{,}827 MB & 58.04 & 47.33\\
Sequence & $\mathcal U(0.75,1)$ & 17{,}072 MB & 58.04 & 49.53\\
Sequence & fixed $0.75$ & 16{,}561 MB & 58.08 & 47.03\\
\midrule
Branch & $q_w{=}0.30$ & 127.7 min & 58.06 & 48.80\\
\rowcolor{tableFocus}
Branch & $q_w{=}0.50$ & 128.4 min & 58.08 & 50.25\\
Branch & $q_w{=}0.70$ & 128.4 min & 57.97 & 48.74\\
\bottomrule
\end{tabular}
\caption{Ablation on Qwen3-4B/HelpSteer3-Preference. Shading marks defaults. Means follow \cref{sec:setup}. Full task scores: \cref{tab:sweep-full}.}
\label{tab:param-sweep}
\end{table}

For Position, increasing $\alpha$ from $0.3$ to $0.9$ raises mean accuracy from $57.90$ to $58.32$, while peak HBM changes little ($11{,}211$--$11{,}580$\,MB) and AH is non-monotone (\cref{tab:sweep-full}). This weak memory dependence agrees with the implementation-dominated attribution above. Sequence cutoffs track input work and HBM but barely affect step time. Fixed $\rho_b=0.75$ and $\mathcal U(0.5,1)$ execute similar token fractions yet differ in temporal weights and randomness, so their score difference cannot isolate the Jensen gap. Branch probabilities yield similar training times ($127.7$--$128.4$\,min) and varying open-ended point estimates. We retain $q_w=0.5$. HBM was not measured across $q_w$. Position step times are $3.29$--$3.41$\,s in its separate paired, interleaved retention profiler; resource comparisons remain within each axis.

\subsection{The Eight-Configuration Factorial}
\label{sec:quality}

In \cref{tab:factorial}, all four Branch configurations finish in $126$--$128$\,min vs. $185$--$188$\,min without Branch, preserving a $31$--$32\%$ time saving across matched contrasts. Position and Sequence retain that benefit when composed with Branch. All 12 single-axis contrasts lower mean accuracy (unrounded contrasts in \appref{app:hs-factorial}); \cref{fig:pareto} shows the resulting time--accuracy tradeoff.
\begin{table}[ht]
\centering
\small
\setlength{\tabcolsep}{3pt}
\begin{tabular}{lrrr}
\toprule
\rowcolor{tableHead}
Method & {\scriptsize\shortstack{Wall\\(min)}} & {\scriptsize\shortstack{Mean\\accuracy (\%)}} & {\scriptsize\shortstack{Mean generation\\score}}\\
\midrule
\rowcolor{tableBase}
Full & 187.8 & \textbf{58.33} & \underline{49.33}\\
Position & 185.9 & 58.06 & 48.38\\
Sequence & 186.4 & 58.04 & 47.33\\
\rowcolor{tableFocus}
Branch & 128.4 & \underline{58.08} & \textbf{50.25}\\
Position+Sequence & 185.0 & 57.97 & 48.85\\
Position+Branch & 127.8 & 57.94 & 48.17\\
Sequence+Branch & \textbf{126.3} & 57.86 & 47.67\\
\shortstack[l]{Position+Sequence\\+Branch} & \underline{127.0} & 57.69 & 48.35\\
\bottomrule
\end{tabular}
\caption{Qwen3-4B factorial, one training and one selection seed. Bold/underline mark best/second-best values, including ties. Full scores: \cref{tab:factorial-tasks}.}
\label{tab:factorial}
\end{table}

Open-ended scores move differently. Branch raises MT and WB point estimates with AH close to Full. Configurations containing Position or Sequence fall below Full on AH, while their MT changes depend on the base configuration. Generation means are descriptive single-seed summaries and do not establish score superiority.

\begin{figure}[ht]
\centering
\includegraphics[width=0.95\columnwidth]{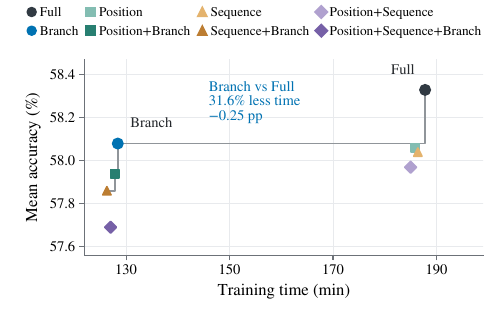}
\caption{Training time versus seven-task mean accuracy for the eight factorial runs. The gray staircase is the observed frontier of single-seed point estimates.}
\label{fig:pareto}
\end{figure}

\subsection{Branch Pruning Across Seeds and Models}
\label{sec:c-validation}

\begin{table}[ht]
\centering
\small
\setlength{\tabcolsep}{3pt}
\begin{tabular}{@{}lrrcc@{}}
\toprule
\rowcolor{tableHead}
Config. & \shortstack{Wall (min)} & Speedup & \shortstack{Mean acc.\\(\%)} & \shortstack{Mean gen.\\score}\\
\midrule
\multicolumn{5}{@{}l}{Qwen3-4B, seed 42}\\
\rowcolor{tableBase}
Full & 187.8 & $1.000\times$ & 58.33 & 49.33\\
\rowcolor{tableFocus}
Branch & \textbf{128.4} & $\boldsymbol{1.463\times}$ & 58.08 & \textbf{50.25}\\
\midrule
\multicolumn{5}{@{}l}{Qwen3-4B, seed 43}\\
\rowcolor{tableBase}
Full & 187.0 & $1.000\times$ & 58.26 & 47.64\\
\rowcolor{tableFocus}
Branch & \textbf{131.0} & $\boldsymbol{1.428\times}$ & 57.81 & \textbf{48.49}\\
\midrule
\multicolumn{5}{@{}l}{Mistral-7B, seed 42}\\
\rowcolor{tableBase}
Full & 211.0 & $1.000\times$ & 44.86 & 33.17\\
\rowcolor{tableFocus}
Branch & \textbf{142.0} & $\boldsymbol{1.486\times}$ & 44.55 & \textbf{33.93}\\
\bottomrule
\end{tabular}
\caption{Full--Branch validation. Speedup uses each pair's Full training time. Bold marks the better value within each pair. Score means follow \cref{sec:setup}; task scores are in \cref{tab:supplement-tasks}.}
\label{tab:c-validation}
\end{table}

The matched Full--Branch trainings in \cref{tab:c-validation} save $29.9$--$32.7\%$ wall-clock across two Qwen seeds and Mistral, with mean accuracy lower by $0.25$--$0.45$\,pp. In the Mistral profile in \cref{tab:branch-profile}, Branch gives $1.507\times$ speedup and saves $1{,}347$\,MB ($6.4\%$) against Full. The time benefit transfers across families.

\begin{table}[ht]
\centering
\small
\setlength{\tabcolsep}{4pt}
\begin{tabular}{lrr}
\toprule
\rowcolor{tableHead}
Configuration & Step (s) & HBM (MB)\\
\midrule
\multicolumn{3}{l}{Qwen3-4B}\\
\rowcolor{tableBase}
Full & 18.204 & 18{,}571\\
\rowcolor{tableFocus}
Branch & \textbf{12.161} & \textbf{17{,}319}\\
\midrule
\multicolumn{3}{l}{Mistral-7B}\\
\rowcolor{tableBase}
Full & 22.438 & 21{,}185\\
\rowcolor{tableFocus}
Branch & \textbf{14.888} & \textbf{19{,}838}\\
\bottomrule
\end{tabular}
\caption{Paired Full--Branch step-time and peak-memory profiles. Bold marks lower values within each model.}
\label{tab:branch-profile}
\end{table}

MT and WB checkpoint means rise in all three comparisons, while Mistral AH falls. These checkpoint scores are descriptive and do not establish superiority or characterize training-seed variability. Full task scores and scoring coverage are in \appref{app:supplement}. Cross-family validation covers Branch only; Position, Sequence, and their compositions are evaluated on Qwen.

\paragraph{Branch-gradient consistency.}
We check the conditional expectation in \cref{prop:branch-unbiased} by enumerating both backward branches. On a small Mistral model in Float64, the probability-weighted average of both branch gradients matches Full to relative L2 error at most $1.7\times10^{-7}$. On Mistral-7B in BFloat16, the gradient cosine is $0.9997$ and the relative L2 error is $0.0236$, against $0.0071$ for repeated computation of the same configuration, which sets the numerical noise floor for this check. These checks disable LoRA dropout and retain gradient checkpointing using dropout $0.05$.

\FloatBarrier
\section{Conclusion}
\label{sec:conclusion}
We separate output scoring, sequence execution, and backward computation to locate where stochastic DPO computation saves resources and how it changes optimization. Most output-memory savings come from the implementation rather than from supervision deletion, the branch-pruning time benefit reproduces across Qwen seeds and Mistral, and both accompany lower mean accuracy with mixed open-ended results, so configuration choice should follow the resource constraint and the acceptable objective change.

\section*{Limitations}
Evaluation covers parameter-efficient text preference training with an online reference on one eight-accelerator setup. Savings may vary with hardware, batch geometry, and reference execution. 

\bibliography{confbib,resubmission_refs}
\appendix


\section{Notation}
\label{app:notation}
We group the main notation below. Symbols specific to individual proofs are defined where introduced. In single-response results, the side index $b$ is suppressed, so $T,k,R,u_t$ denote the corresponding side quantities.

\paragraph{Preference pairs and loss.}
$\train$ is the preference dataset of prompts $\vx$, chosen responses $\vy_w$, and rejected responses $\vy_l$. The policy $\pol$ has trainable parameters $\vtheta$, and $\refpol$ is the fixed reference. For side $b\in\{w,l\}$, $T_b$ is the response length, $\evy_{b,t}$ the target token at position $t$, and $\vy_{b,<t}$ its response prefix. The token log-ratio $u_{b,t}$ gives the mean reward $R_b$ and margin $\Delta=R_w-R_l$. The scale $\beta$ defines $\ell(d)=\log(1+e^{-\beta d})$, with per-pair loss $\Ls_{\dpo}=\ell(\Delta)$ (\cref{sec:problem}).

\paragraph{Three pruning axes.}
Sequence pruning uses cutoff fraction $\rho_b$ to retain prefix $\tilde\vy_b$ of length $M_b$, with minimum fraction $\rho_{\min}$. Position pruning selects $\sS_b\subseteq\sI_{M_b}=\{1,\ldots,M_b\}$ with count $k_b$ controlled by retention $\alpha$. Branch pruning samples backward branch $b^\star$ with probability $q_b$ (\cref{sec:method-position,sec:method-prefix,sec:method-branch}).

\paragraph{Sampling and execution.}
$\prefixsample$, $\positionsample$, and $\branchsample$ denote the prefix, position, and branch sampling laws. The active backward branches form $\sB$, and $q^\star$ supplies the probability correction. The operator $\logprobs$ returns scored log-probabilities, with Boolean flag $g$ controlling gradient recording. Its policy and reference outputs are $\vz_b$ and $\vz_b^{\mathrm{ref}}$, and $\Ls$ is the step loss. We write $\sigmoid$ for sigmoid, $\sg$ for stop-gradient, and $V$ for vocabulary size (\cref{sec:problem,sec:method-branch,sec:composition}).

\paragraph{Analysis.}
$\E$ and $\Var$ refer to the indicated sampling. $\sU_k$ is a uniform size-$k$ position subset, $\vphi_t$ a token log-probability gradient, and $\Delta_U$ and $\Delta_{\mathrm{full}}$ the sampled and full margins (\cref{prop:random-loss}). The objectives $F,F_A,F_P,F_{AP}$ correspond to full responses, position pruning, sequence pruning, and their composition, with minimizers $\vtheta^*,\vtheta_A^*,\vtheta_P^*,\vtheta_{AP}^*$. The analysis uses token-score and gradient bounds $U,G$, variance bound $\bar v$, and length bounds $T_0,M_0$ (\cref{sec:method-position,sec:composition}).

\section{Proofs}
\label{app:proofs}
This appendix collects the proofs of every stated result: Proposition~\ref{prop:random-loss} (\cref{app:prop-random}), Lemma~\ref{lem:temporal-weight} (\cref{app:proof-temporal}), Lemma~\ref{lem:prefix-gap} (\cref{app:proof-prefix}), Proposition~\ref{prop:branch-unbiased} (\cref{app:proof-branch}).

\subsection{Proof of Proposition~\ref{prop:random-loss}}
\label{app:prop-random}
\begin{proof}[Proof of \cref{prop:random-loss}]
Writing $\mathbb{1}\{E\}$ for the indicator of event $E$, \cref{eq:random-unbiased} follows from $\E[\mathbb{1}\{t\in\sU_k\}]=k/T$ and sampling without replacement. Jensen's inequality gives the lower bound in \cref{eq:random-loss-gap}. Since $0\leq\ell''(d)\leq\beta^2/4$, a second-order upper bound around $\Delta_{\mathrm{full}}$ gives the upper bound. The variance is the finite-population variance of a without-replacement sample, summed over independent sides.
\end{proof}

The finite-population identities behind the proposition are as follows. Let $I_t$ indicate membership in a uniform size-$k$ subset of $T$ positions. Then $\E I_t=k/T$, and for $i\ne j$, $\E[I_iI_j]=k(k-1)/[T(T-1)]$. These identities establish \cref{eq:random-unbiased} and the variance in \cref{prop:random-loss}. For vector gradients and $T>1$, let $\bar\vphi=T^{-1}\sum_t\vphi_t$ be their population mean and $\mSigma_\phi$ their population covariance. $\Cov$ denotes covariance. Then
\begin{align}
\Cov\left[\frac1k\sum_t I_t\vphi_t\right]
=\frac{T-k}{k(T-1)}\mSigma_\phi,\\
\mSigma_\phi=\frac1T\sum_t(\vphi_t-\bar\vphi)(\vphi_t-\bar\vphi)^\top,\nonumber
\end{align}
For $T=1$, the covariance is zero. These identities concern the side-reward gradient. The sampled loss gradient is $-\beta\sigmoid(-\beta\Delta_U)\nabla_{\vtheta}\Delta_U$. Its mask-dependent multiplier means that reward-gradient unbiasedness alone does not imply full-loss-gradient unbiasedness. \Cref{eq:random-loss-gap} quantifies the expected loss difference.

\paragraph{Remark (target-logit diagnostics).}
Let $p_t$ be the current-policy target-token probability at position $t$, and $\sS$ a selected set. The target-logit diagnostic $\energy_z$ (subscript $z$ denotes the target logit) is $\energy_z(\sS)=\frac1{|\sS|^2}\sum_{t\in\sS}(1-p_t)^2$ on a fixed response. For the full position set $\sI_T$, linearity gives
\begin{align}
\E[\energy_z(\sU_k)]&=\frac1{kT}\sum_t(1-p_t)^2,\\
\energy_z(\sI_T)&=\frac1{T^2}\sum_t(1-p_t)^2.\nonumber
\end{align}
Their ratio is $T/k$, not $k/T$. If a diagnostic instead reports the retained numerator fraction $f(\sS)=\sum_{t\in\sS}(1-p_t)^2/\sum_{t=1}^T(1-p_t)^2$, its normalized reward-derivative ratio is $(T/k)^2f(\sS)$ for that response. With varying response lengths and rounded counts, conversion must occur per response before aggregation. Neither quantity measures conditional KL or downstream quality.

\paragraph{Remark (alternative denominator).}
Keeping the $1/T_b$ denominator would scale the selected mean by the realized fraction $k_b/T_b$. At a common realized fraction $\gamma=k_b/T_b$ on both sides, this equals selected-mean normalization at temperature $\gamma\beta$. We use the $1/k_b$ mean throughout.

\subsection{Proof of Lemma~\ref{lem:temporal-weight}}
\label{app:proof-temporal}
\begin{proof}[Proof of \cref{lem:temporal-weight}]
Exchange the order of summation:
\begin{align}
\E_M[R_M]&=\sum_{m=1}^{T}q(m)\,\frac1m\sum_{t=1}^{m}u_t\nonumber\\
&=\sum_{t=1}^{T}u_t\sum_{m=t}^{T}\frac{q(m)}{m}
=\sum_{t=1}^{T}w_tu_t.
\end{align}
For $1\leq t<T$, $w_t-w_{t+1}=q(t)/t\geq0$, and $w_T=q(T)/T\geq0$. Also, $\sum_{t=1}^T w_t=\sum_{m=1}^T q(m)\sum_{t=1}^m 1/m=1$. Conversely, for a nonnegative, normalized, non-increasing vector $\vw=(w_1,\ldots,w_T)$, set $w_{T+1}:=0$ and $q(t)=t(w_t-w_{t+1})$. Then $q(t)\geq0$ and $\sum_{t=1}^T q(t)=\sum_{t=1}^T t(w_t-w_{t+1})=\sum_{t=1}^T w_t=1$. Substitution recovers $w_t=\sum_{m=t}^Tq(m)/m$. Uniform weights $w_t=1/T$ give $q(t)=0$ for $t<T$ and $q(T)=1$, so $M=T$ almost surely, recovering the fixed-cutoff case $\rho_b\equiv1$.
\end{proof}

\subsection{Proof of Lemma~\ref{lem:prefix-gap}}
\label{app:proof-prefix}
\begin{proof}[Proof of \cref{lem:prefix-gap}]
$\ell(d)=\log(1+e^{-\beta d})$ satisfies $\ell'(d)=-\beta\sigmoid(-\beta d)$ and $\ell''(d)=\beta^2\sigmoid(\beta d)\sigmoid(-\beta d)$. For $\beta>0$, $0<\ell''\leq\beta^2/4$ and $\ell'<0$, so $\ell$ is strictly convex and strictly decreasing. For $\beta=0$, $\ell\equiv\log2$ with $\ell'=\ell''=0$. Assume $\beta>0$ unless a degenerate case is named. Jensen's inequality gives $\E[\ell(\Delta_M)]\geq\ell(\mu_{\mathrm{pref}})$, hence $J_{\mathrm{pref}}:=\Ls_{\mathrm{pref}}-\ell(\mu_{\mathrm{pref}})\geq0$. For the upper bound, expand $\ell$ to second order around $\mu_{\mathrm{pref}}$: for each scalar margin $x$, $\ell(x)=\ell(\mu_{\mathrm{pref}})+\ell'(\mu_{\mathrm{pref}})(x-\mu_{\mathrm{pref}})+\tfrac12\ell''(\xi)(x-\mu_{\mathrm{pref}})^2$ for some $\xi$ between $x$ and $\mu_{\mathrm{pref}}$. Taking the expectation removes the linear term, and $\ell''\leq\beta^2/4$ gives
\begin{equation}
J_{\mathrm{pref}}=\tfrac12\,\E[\ell''(\xi)(\Delta_M-\mu_{\mathrm{pref}})^2]\leq\tfrac{\beta^2}{8}\Var(\Delta_M).
\end{equation}
Subtracting the full-response loss from \cref{eq:prefix-gap} gives
\begin{equation}
\begin{aligned}
\Ls_{\mathrm{pref}}-\ell(\Delta_{\mathrm{full}})
&=\ell(\mu_{\mathrm{pref}})-\ell(\Delta_{\mathrm{full}})\\
&\quad+J_{\mathrm{pref}}.
\end{aligned}
\label{eq:prefix-decomp}
\end{equation}
Because $\ell$ is strictly decreasing, $\ell(\mu_{\mathrm{pref}})-\ell(\Delta_{\mathrm{full}})$ is positive precisely when the temporally reweighted mean margin $\mu_{\mathrm{pref}}$ is below the full-response margin, which is a data- and distribution-dependent event and therefore carries no fixed sign. \Cref{eq:temporal-weight} gives $\mu_{\mathrm{pref}}=\sum_t w_t^{(w)}u_{w,t}-\sum_t w_t^{(l)}u_{l,t}$, with the side-specific weights $w_t^{(b)}$. Two degeneracies: a deterministic cutoff makes $\Delta_M$ constant, so $\Var(\Delta_M)=0$ and $J_{\mathrm{pref}}=0$. At $\beta=0$ both terms of \cref{eq:prefix-decomp} vanish.

Finally, the cutoff variables $M_w,M_l$ are drawn independently of $\vtheta$ and have finite support, so the expectation in $\Ls_{\mathrm{pref}}(\vtheta)=\E[\ell(\Delta_M(\vtheta))]$ is a finite sum and differentiation passes through it exactly:
\begin{equation}
\nabla_{\vtheta}\Ls_{\mathrm{pref}}=\E\bigl[\nabla_{\vtheta}\ell(\Delta_M)\bigr].
\end{equation}
Thus $\nabla_{\vtheta}\ell(\Delta_M)$ is an unbiased stochastic gradient of $\Ls_{\mathrm{pref}}$ and need not coincide in expectation with the full-response gradient. The temporal shift and Jensen gap in \cref{eq:prefix-decomp} can cancel at a fixed policy. Equality of loss values at one point does not establish equality of gradients.
\end{proof}

\subsection{Proof of Proposition~\ref{prop:branch-unbiased}}
\label{app:proof-branch}
\begin{proof}[Proof of \cref{prop:branch-unbiased}]
Write $\Ls_w^{\text{sg}}=\softplus(-\beta(R_w-\sg(R_l)))$. Because the stop-gradient blocks the flow into $R_l$, and because $\lambda=\beta\sigmoid(-\beta(R_w-R_l))$ is evaluated at the current parameters as a chain-rule coefficient (it is differentiated no further),
\begin{align}
\nabla_{\vtheta}\Ls_w^{\text{sg}}
&=\sigmoid\bigl(-\beta(R_w-R_l)\bigr)(-\beta)\nabla_{\vtheta}R_w\nonumber\\
&=-\lambda\nabla_{\vtheta}R_w,
\end{align}
and symmetrically $\nabla_{\vtheta}\Ls_l^{\text{sg}}=+\lambda\nabla_{\vtheta}R_l$. With the $1/q_b$ importance factors, $\E_{b^\star}[\nabla\widehat\Ls]=q_w\cdot(-\lambda/q_w)\nabla R_w+q_l\cdot(\lambda/q_l)\nabla R_l=-\lambda\nabla R_w+\lambda\nabla R_l=\nabla_{\vtheta}\log(1+e^{-\beta(R_w-R_l)})$. All quantities are those of the sampled response positions and prefixes: with sequence pruning active, $R_b$ is the prefix mean over $M_b$. With position pruning active, it is the masked mean over $\sS_b$, and the gradients flow only through the selected positions. When position and sequence pruning are disabled, the sampled computation is the full-response objective and the claim recovers the length-normalized loss in \cref{eq:dpo-loss}.
\end{proof}

\paragraph{Remark (branch-estimator variance).}
With $\|\cdot\|$ denoting the Euclidean norm, for fixed sampled rewards the single-branch estimator's second moment is $\lambda^2\bigl(\|\nabla R_w\|^2/q_w+\|\nabla R_l\|^2/q_l\bigr)$. Its variance-optimal probability therefore balances the branch gradient norms, the usual importance-sampling criterion \citep{katharopoulos2018importance}. We use $q_w=0.5$ as a symmetric default. The $q_w\in\{0.3,0.5,0.7\}$ sweep (\cref{tab:param-sweep}) has similar wall-clock times and differing quality point estimates.

\subsection{Proof of \texorpdfstring{\cref{cor:position-optimality}}{Corollary \ref{cor:position-optimality}} and the retention rule}
\label{app:proof-optimality}
\begin{proof}[Proof of \cref{cor:position-optimality}]
Since $\ell$ is convex, Jensen's inequality gives $\E_S[\ell(\Delta_U)]\geq\ell(\E_S[\Delta_U])=\ell(\Delta_{\mathrm{full}})$ at every $\vtheta$, so $F_A\geq F$ pointwise. Hence
\begin{equation}
\begin{aligned}
F(\vtheta_A^*)&\leq F_A(\vtheta_A^*)\leq F_A(\vtheta^*)\\
&\leq F(\vtheta^*)+\frac{\beta^2}{8}\E_{\train}\bigl[\Var_S(\Delta_U)\bigr],
\end{aligned}
\end{equation}
where the first step uses $F\leq F_A$, the second the optimality of $\vtheta_A^*$ for $F_A$, and the third \cref{eq:random-loss-gap} evaluated at $\vtheta^*$. Subtracting $F(\vtheta^*)$ gives \cref{eq:position-optimality}.

For \cref{eq:rA}, $k_b\geq\alpha T_b$ implies $1/k_b\leq1/(\alpha T_b)$ and $T_b-k_b\leq(1-\alpha)T_b$, while $v_b\leq\bar v$ and $T_b\geq T_0$ give
\begin{equation}
\begin{aligned}
\Var_S(\Delta_U)&=\sum_{b:T_b>1}\frac{T_b-k_b}{k_b(T_b-1)}v_b\\
&\leq\sum_{b}\frac{(1-\alpha)T_b}{\alpha T_b(T_0-1)}\bar v
=2\bar v\,r_A .
\end{aligned}
\end{equation}
The display is per response, so averaging over the dataset gives $\E_{\train}[\Var_S(\Delta_U)]\leq2\bar v\,r_A$. Requiring $\frac{\beta^2}{8}\E_{\train}[\Var_S(\Delta_U)]\leq\frac{\beta^2\bar v}{4}r_A\leq\varepsilon_F$ and solving for $\alpha$ yields \cref{eq:alpha-rule}.
\end{proof}

\subsection{Proof of \texorpdfstring{\cref{eq:gradient-bias}}{Eq. (\ref{eq:gradient-bias})} and the stationary displacement}
\label{app:proof-bias}
\begin{proof}
Fix one preference pair and abbreviate $d=\Delta_{\mathrm{full}}$, $g=\nabla_{\vtheta}\Delta_{\mathrm{full}}$, $\delta=\Delta_U-d$, and $e=\nabla_{\vtheta}\Delta_U-g$. Unbiasedness gives $\E\delta=0$ and $\E e=0$, so $V_u=\E\delta^2$ and $V_g=\E\|e\|^2$. Since $\nabla_{\vtheta}\ell(\Delta_U)=\ell'(d+\delta)(g+e)$ and $\nabla_{\vtheta}\ell(d)=\ell'(d)g$,
\begin{equation}
\begin{aligned}
\nabla_{\vtheta}\ell(\Delta_U)-\nabla_{\vtheta}\ell(d)
={}&\bigl[\ell'(d+\delta)-\ell'(d)\bigr]g\\
&+\ell'(d+\delta)e .
\end{aligned}
\end{equation}
The linear term in the first bracket has zero mean, so only $\frac12\ell'''(\xi)\delta^2$ survives, and its norm is at most $\frac12\sup|\ell'''|\,V_u\|g\|$. Because $\ell'''=\beta^3s(1-s)(1-2s)$ with $s=\sigmoid(\beta d)$ satisfies $\max_s|s(1-s)(1-2s)|=1/(6\sqrt3)$, we have $\sup|\ell'''|=\beta^3/(6\sqrt3)$. For the second term, centering with $\E e=0$ and applying Cauchy--Schwarz with $\sup|\ell''|=\beta^2/4$ bounds it by $\frac{\beta^2}{4}\sqrt{V_uV_g}$. Therefore
\begin{equation}
\begin{aligned}
\|\E\nabla_{\vtheta}\ell(\Delta_U)-\nabla_{\vtheta}\ell(d)\|
\leq{}&\frac{\beta^3}{12\sqrt3}V_u\|g\|\\
&+\frac{\beta^2}{4}\sqrt{V_uV_g}.
\end{aligned}
\end{equation}
Averaging over pairs replaces the single-pair quantities by their population counterparts. Substituting $V_u\leq2\bar v r_A$, $V_g\leq2G^2r_A$, and $\|g\|\leq2G$ bounds the two terms by $\frac{\beta^3G\bar v}{3\sqrt3}r_A$ and $\frac{\beta^2G\sqrt{\bar v}}{2}r_A$, so that
\begin{equation}
\begin{aligned}
\|\nabla F_A(\vtheta)-\nabla F(\vtheta)\|&\leq C_A\,r_A,\\
C_A&:=\frac{\beta^3G\bar v}{3\sqrt3}+\frac{\beta^2G\sqrt{\bar v}}{2}.
\end{aligned}
\label{eq:gradient-bias}
\end{equation}

Now let $\vtheta_A$ be an interior stationary point of $F_A$ and $\vtheta^*$ an interior stationary point of $F$, both in the analyzed region, and suppose $\nabla^2F\succeq\mu I$ with $\mu>0$ along the segment joining them. Then $\|\nabla F(\vtheta_A)\|=\|\nabla F(\vtheta_A)-\nabla F_A(\vtheta_A)\|\leq C_Ar_A$. Integrating the Hessian along that segment gives $\langle\nabla F(\vtheta_A),\vtheta_A-\vtheta^*\rangle\geq\mu\|\vtheta_A-\vtheta^*\|^2$; Cauchy--Schwarz together with the previous display yields $\|\vtheta_A-\vtheta^*\|\leq C_Ar_A/\mu$.
\end{proof}

\subsection{Curvature constants from the margin geometry}
\label{app:proof-curvature}
Suppose $|\beta\Delta_{\vtheta}|\leq C$ on the analyzed convex region, $\E_{\train}[\nabla_{\vtheta}\Delta\nabla_{\vtheta}\Delta^{\!\top}]\succeq\nu I$, and $\|\nabla^2_{\vtheta}\Delta\|_{\mathrm{op}}\leq H$. With $m_C=\sigmoid(C)\sigmoid(-C)$, the chain rule gives
\begin{equation}
\begin{aligned}
\nabla^2F&=\E\bigl[\ell''(\Delta)\nabla_{\vtheta}\Delta\nabla_{\vtheta}\Delta^{\!\top}+\ell'(\Delta)\nabla^2_{\vtheta}\Delta\bigr]\\
&\succeq\bigl(\beta^2m_C\nu-\beta H\bigr)I ,
\end{aligned}
\end{equation}
because $\ell''\geq\beta^2m_C$ on the region and $|\ell'|\leq\beta$. Hence $\beta^2m_C\nu>\beta H$ is a sufficient stability condition and one may take $\mu=\beta^2m_C\nu-\beta H$. In a linear-margin model $H=0$, so $\mu>0$ as soon as the comparison matrix has positive smallest eigenvalue $\nu$.

\subsection{Proof of \texorpdfstring{\cref{eq:prefix-bias}}{Eq. (\ref{eq:prefix-bias})}}
\label{app:proof-prefix-bias}
\begin{proof}
Let $h_b=1-M_b/T_b$. Since $M_b=\max(1,\ceil{\rho_bT_b})$ and $\rho_b\sim\mathcal U(\rho_{\min},1)$,
\begin{equation}
\begin{aligned}
\E[h_b]=1-\frac{\E M_b}{T_b}&\leq1-\frac{\E\rho_bT_b}{T_b}\\
&=1-\frac{1+\rho_{\min}}{2}=\frac{1-\rho_{\min}}{2}.
\end{aligned}
\end{equation}
For one realization, split the reward difference into its retained and removed parts:
\begin{equation}
\begin{aligned}
R_{b,M_b}-R_b={}&\Bigl(\frac1{M_b}-\frac1{T_b}\Bigr)\sum_{t\leq M_b}u_{b,t}\\
&-\frac1{T_b}\sum_{t>M_b}u_{b,t},
\end{aligned}
\end{equation}
so $|R_{b,M_b}-R_b|\leq(1-\frac{M_b}{T_b})U+(1-\frac{M_b}{T_b})U=2Uh_b$. The same argument with $\|\nabla_{\vtheta}u_{b,t}\|\leq G$ gives $\|\nabla_{\vtheta}R_{b,M_b}-\nabla_{\vtheta}R_b\|\leq2Gh_b$. The margin difference is the difference of the two sides, so $|\Delta_M-\Delta_{\mathrm{full}}|\leq2U(h_w+h_l)$ and $\|\nabla_{\vtheta}\Delta_M-\nabla_{\vtheta}\Delta_{\mathrm{full}}\|\leq2G(h_w+h_l)$.

The loss is $\beta$-Lipschitz, so $|\ell(\Delta_M)-\ell(\Delta_{\mathrm{full}})|\leq2\beta U(h_w+h_l)$; taking expectations gives the loss bound of \cref{eq:prefix-bias}. For the gradient, decompose
\begin{equation}
\begin{aligned}
&\ell'(\Delta_M)\nabla_{\vtheta}\Delta_M-\ell'(\Delta_{\mathrm{full}})\nabla_{\vtheta}\Delta_{\mathrm{full}}\\
&=\bigl[\ell'(\Delta_M)-\ell'(\Delta_{\mathrm{full}})\bigr]\nabla_{\vtheta}\Delta_M\\
&\quad+\ell'(\Delta_{\mathrm{full}})\bigl[\nabla_{\vtheta}\Delta_M-\nabla_{\vtheta}\Delta_{\mathrm{full}}\bigr].
\end{aligned}
\end{equation}
Each side's mean-reward gradient is an average of token gradients of norm at most $G$, so $\|\nabla_{\vtheta}\Delta_M\|\leq2G$ at every truncation, with no truncation-dependent factor. The first term is therefore at most $\frac{\beta^2}{4}\cdot2U(h_w+h_l)\cdot2G=\beta^2UG(h_w+h_l)$, using $\sup|\ell''|=\beta^2/4$, and the second at most $\beta\cdot2G(h_w+h_l)$. Taking expectations and using $\E[h_w+h_l]\leq1-\rho_{\min}$ from above gives an expected gradient bias of at most $G(\beta^2U+2\beta)(1-\rho_{\min})$, which is the constant $C_P$ used in \cref{app:proof-composed}.

For the optimality statement, $F\leq F_P+\Delta$ with $\Delta=2\beta U(1-\rho_{\min})$, so $F(\vtheta_P^*)\leq F_P(\vtheta_P^*)+\Delta\leq F_P(\vtheta^*)+\Delta\leq F(\vtheta^*)+2\Delta$. The displacement bound follows from the segment argument of \cref{app:proof-bias} applied to the gradient bound above, and inherits the same requirements: an interior stationary point of each objective lying in the analyzed region, and $\nabla^2F\succeq\mu I$ with $\mu>0$ along the joining segment.

The loss bound is a scaling statement, not a guarantee at the tested setting. Since $|R_b|\leq U$ we have $|\Delta|\leq2U$, so the total range of the loss on the relevant interval is $\ell(-2U)-\ell(2U)=2\beta U$. At $\rho_{\min}=1/2$ the optimality-gap bound $4\beta U(1-\rho_{\min})$ reaches $2\beta U$, exactly that range, so it is vacuous at the tested cutoff; the loss bound itself is $2\beta U(1-\rho_{\min})=\beta U$, half that range. Both become restrictive only for cutoffs closer to one.
\end{proof}

\subsection{Proof of \texorpdfstring{\cref{prop:branch-variance}}{Proposition \ref{prop:branch-variance}} and \texorpdfstring{\cref{eq:branch-rate}}{Eq. (\ref{eq:branch-rate})}}
\label{app:proof-branch-rate}
\begin{proof}
With $\vc_w,\vc_l$ as defined in \cref{prop:branch-variance}, $\E[\hat\vg]=\vc_w+\vc_l$ follows from \cref{prop:branch-unbiased}, and
\begin{equation}
\E\|\hat\vg\|^2=q_w\frac{\|\vc_w\|^2}{q_w^2}+q_l\frac{\|\vc_l\|^2}{q_l^2}=\frac{\|\vc_w\|^2}{q_w}+\frac{\|\vc_l\|^2}{q_l},
\end{equation}
so $\Var(\hat\vg)=\E\|\hat\vg\|^2-\|\vc_w+\vc_l\|^2$ is \cref{eq:branch-variance}. At $q_w=q_l=1/2$ it equals $2\|\vc_w\|^2+2\|\vc_l\|^2-\|\vc_w+\vc_l\|^2=\|\vc_w-\vc_l\|^2$. Because $|\lambda|=\beta\sigmoid(-\beta\Delta)\leq\beta$ and $\|\nabla_{\vtheta}R_b\|\leq G$, we have $\|\vc_w\|,\|\vc_l\|\leq\beta G$, hence $\Var(\hat\vg)\leq4\beta^2G^2$, and at the symmetric default also $\|\hat\vg\|=\|\vc_w\|/q_w\leq2\beta G$. Averaging $m$ estimators divides the variance by $m$ when the $m$ draws are independent in their preference pairs, response positions, prefixes and branches; if only the branch is re-drawn, only the branch component is reduced.

For the rate, take the symmetric default $q_w=q_l=1/2$ so that each step's noise has second moment at most $4\beta^2G^2/m$. Let $\vtheta_\Phi^*$ minimize $\Phi$ and $D_n=\E\|\vtheta_n-\vtheta_\Phi^*\|^2$. Projection is non-expansive and $\vtheta_\Phi^*\in\Theta$, so
\begin{equation}
\begin{aligned}
D_{n+1}\leq{}&D_n-2\eta_n\E\langle\nabla\Phi(\vtheta_n),\vtheta_n-\vtheta_\Phi^*\rangle\\
&+\eta_n^2\Bigl(\E\|\nabla\Phi(\vtheta_n)\|^2+\frac{4\beta^2G^2}{m}\Bigr).
\end{aligned}
\end{equation}
Let $A$ be the average Hessian along the segment from $\vtheta_\Phi^*$ to $\vtheta_n$. Strong convexity gives $A\succeq\mu I$ and smoothness gives $A\preceq LI$, so $A^2\preceq LA$. Since $\nabla\Phi(\vtheta_n)=A(\vtheta_n-\vtheta_\Phi^*)$,
\begin{equation}
\|\nabla\Phi(\vtheta_n)\|^2\leq L\langle\nabla\Phi(\vtheta_n),\vtheta_n-\vtheta_\Phi^*\rangle .
\end{equation}
Substituting this and using $\langle\nabla\Phi,\vtheta-\vtheta_\Phi^*\rangle\geq\mu D_n$ together with $\eta_nL\leq1$, which $t_0\geq2L/\mu$ ensures,
\begin{equation}
D_{n+1}\leq(1-\mu\eta_n)D_n+\frac{4\beta^2G^2}{m}\eta_n^2 .
\end{equation}
Insert $\eta_n=2/[\mu(n+t_0)]$, write $s=n+t_0$, and suppose $D_n\leq A_B/s$. Then
\begin{equation}
\begin{aligned}
D_{n+1}&\leq\frac{A_B(s-2)}{s^2}+\frac{16\beta^2G^2}{\mu^2ms^2}\\
&\leq\frac{A_B(s-1)}{s^2}\leq\frac{A_B}{s+1},
\end{aligned}
\end{equation}
where the middle step uses $A_B\geq16\beta^2G^2/(\mu^2m)$ and the last uses $(s-1)(s+1)\leq s^2$. Since $D_0\leq A_B/t_0$, induction gives
\begin{equation}
\begin{aligned}
\E\|\vtheta_N-\vtheta_\Phi^*\|^2&\leq\frac{A_B}{N+t_0},\\
A_B&:=\max\Bigl\{t_0\|\vtheta_0-\vtheta_\Phi^*\|^2,\\&\qquad\frac{16\beta^2G^2}{\mu^2m}\Bigr\}
\end{aligned}
\label{eq:branch-rate}
\end{equation}
for $\eta_n=2/[\mu(n+t_0)]$ and $t_0\geq\max\{2,2L/\mu\}$. Markov's inequality applied to $\|\vtheta_N-\vtheta_\Phi^*\|^2$ gives $\Pr(\|\vtheta_N-\vtheta_\Phi^*\|>\varepsilon_\theta)\leq A_B/[(N+t_0)\varepsilon_\theta^2]$.
\end{proof}

\subsection{Proof of the composed bias and error budget}
\label{app:proof-composed}
\begin{proof}
Write $r_{A|P}=(1-\alpha)/[\alpha(M_0-1)]$.
Write $\nabla F_{AP}-\nabla F=(\nabla F_{AP}-\nabla F_P)+(\nabla F_P-\nabla F)$. Conditioned on the executed lengths $(M_w,M_l)$, the first difference is the position-pruning bias of \cref{eq:gradient-bias} evaluated at those lengths. A fixed prefix can have a larger token-score variance than the whole response, so this difference must be averaged over the cutoff law rather than bounded prefix by prefix. Two facts supply that average. First, the within-prefix sum of squares cannot exceed the whole-response sum of squares: splitting the total sum of squares into the prefix, the suffix and the part between their means gives
\begin{equation}
\begin{aligned}
\sum_{t\leq T}(u_t-\bar u_T)^2
={}&\sum_{t\leq M}(u_t-\bar u_M)^2\\
&+\sum_{t>M}(u_t-\bar u_{>M})^2\\
&+\frac{M(T-M)}{T}\bigl(\bar u_M-\bar u_{>M}\bigr)^2,
\end{aligned}
\end{equation}
so $Mv_M\leq Tv_T$ and hence $v_M\leq(T/M)v_T$, with the suffix terms read as zero when $M=T$. Second, for $M_b=\max(1,\ceil{\rho_bT_b})$ with $\rho_b\sim\mathcal U(\rho_{\min},1)$, prefixes of length one carry zero variance and the remaining lengths satisfy $\E[\mathbf 1_{\{M_b>1\}}T_b/(M_b(M_b-1))]\leq1/(M_0-1)$ with $M_0=\max(2,\ceil{\rho_{\min}T_0})$. Combining the two bounds the prefix-averaged position-pruning bias by $C_Ar_{A|P}$, the factor two accounting for the two sides. The second difference is the prefix bias of \cref{eq:prefix-bias}, contributing $C_P(1-\rho_{\min})$ with $C_P=G(\beta^2U+2\beta)$. Adding the two gives
\begin{equation}
\begin{aligned}
\|\nabla F_{AP}-\nabla F\|&\leq C_Ar_{A|P}+C_P(1-\rho_{\min}),\\
C_P&:=G(\beta^2U+2\beta).
\end{aligned}
\label{eq:composed-bias}
\end{equation}

If the composed objective has interior minimizer $\vtheta_{AP}^*$ and $F$ has curvature at least $\mu_F$ on the relevant segment, the segment argument of \cref{app:proof-bias} gives
\begin{equation}
\|\vtheta_{AP}^*-\vtheta^*\|\leq\frac{C_Ar_{A|P}+C_P(1-\rho_{\min})}{\mu_F}.
\end{equation}
The $L^2$ triangle inequality $\|\vtheta_N-\vtheta^*\|\leq\|\vtheta_N-\vtheta_{AP}^*\|+\|\vtheta_{AP}^*-\vtheta^*\|$, combined with \cref{eq:branch-rate} applied to the composed objective, yields
\begin{equation}
\begin{aligned}
\sqrt{\E\|\vtheta_N-\vtheta^*\|^2}
\;\leq\;&\frac{C_Ar_{A|P}+C_P(1-\rho_{\min})}{\mu_F}\\
&+\sqrt{\frac{A_{B,AP}}{N+t_{0,AP}}},
\end{aligned}
\label{eq:error-budget}
\end{equation}
where $\mu_F$ bounds the curvature of $F$ on the relevant segment and $A_{B,AP},t_{0,AP}$ are the constants of \cref{eq:branch-rate} evaluated for the composed objective.

For the composed optimal loss, $F_{AP}\geq F_P$ by Jensen's inequality in the position masks and $F\leq F_P+2\beta U(1-\rho_{\min})$. Chaining these as in \cref{app:proof-optimality} gives
\begin{equation}
F(\vtheta_{AP}^*)-F(\vtheta^*)\leq\frac{\beta^2\bar v}{4}r_{A|P}+4\beta U(1-\rho_{\min}).
\end{equation}
\end{proof}

\section{Shared-Parameter Analysis}
\label{app:conf-analysis}
This appendix expands the shared-parameter coupling discussed in the introduction.

\subsection{Shared-Parameter Updates and Signed Token Movement}
\label{app:coupled}
Let $i$ index token positions across a pair, $p_i$ be the current-policy target-token probability, and $\vphi_i=\nabla_{\vtheta}\log p_i$ its log-probability gradient. Set $c_i=1/k_w$ for a selected chosen token, $c_i=-1/k_l$ for a selected rejected token, and $c_i=0$ otherwise. Within a region of fixed mask membership,
\begin{equation}
\vg=\nabla_{\vtheta}\Delta=\sum_i c_i\vphi_i.
\end{equation}
For a small gradient-descent step of size $\eta>0$ on the pair loss, the parameter change is $\eta\beta\sigmoid(-\beta\Delta)\vg$. Writing $\delta$ for the resulting change, to first order
\begin{align}
\delta\log p_i&=\eta\beta\sigmoid(-\beta\Delta)\vphi_i^\top \vg+\mathcal{O}(\eta^2),\label{eq:coupled-update}\\
\delta\Delta&=\eta\beta\sigmoid(-\beta\Delta)\|\vg\|^2+\mathcal{O}(\eta^2).\nonumber
\end{align}
The pair margin increases to first order, yet an individual chosen probability can decrease if $\vphi_i^\top \vg<0$. Token confidence alone does not determine these inner products \citep{yuan2025a,razin2025unintentional}. Minibatching and AdamW further change the actual update. The single-pair gradient calculation identifies the missing interaction rather than modeling that optimizer exactly.

\section{Experimental Settings and Results}
\label{app:helpsteer}
This appendix documents the experimental settings and additional results: the UltraFeedback motivation experiments (\cref{app:original-results}), the HelpSteer3 corpus and recipe (\cref{app:hs-recipe}), the retention-sweep geometry (\cref{app:hs-sweep}), the factorial and parameter sweeps (\cref{app:hs-factorial}), the full per-task tables (\cref{app:task-tables}), the detailed HelpSteer3 tables (\cref{app:main-tables}), and the supplemental attribution and Full--Branch validation (\cref{app:supplement}).

\subsection{The UltraFeedback Experiments}
\label{app:original-results}
The UltraFeedback collection uses sequence length 512, prompt length 256, one training epoch, AdamW, cosine decay, BFloat16 computation, 4-bit loading, gradient checkpointing, and LoRA on attention and MLP projections. The reference is the frozen policy initialization. Qwen thinking mode is disabled for evaluation. Experiments run on Ubuntu 22.04 with two NVIDIA RTX 5090 GPUs (32\,GiB each), Python 3.12, PyTorch 2.8.0, and CUDA 12.8. Masks are selected independently on the two sides each step, with the chosen side ranked by current-policy confidence (Low) or uniformly (Random), and gradients bypass the discrete selection.

All three models use LoRA dropout $0.05$. The principal comparison uses retention $0.8$ and training seed 42 for each rule. Pythia-2.8B uses learning rate $5\times10^{-7}$, DPO $\beta=0.1$, LoRA rank 16 and scaling alpha 32, per-device batch size 2, 16 gradient accumulation steps, and warmup fraction $0.10$. Qwen3-4B uses learning rate $1.5\times10^{-6}$, $\beta=0.5$, LoRA rank 64 and scaling alpha 128, per-device batch size 2, 16 gradient accumulation steps, and warmup fraction $0.08$. Qwen3-8B uses learning rate $7.5\times10^{-7}$, $\beta=0.5$, LoRA rank 32 and scaling alpha 64, per-device batch size 1, 32 gradient accumulation steps, and warmup fraction $0.10$.

The tasks include ARC Challenge \citep{clark2018think}, BBH \citep{suzgun2022challenging}, MMLU-Pro \citep{wang2024mmlupro}, MuSR \citep{sprague2024mus}, TruthfulQA MC1 \citep{lin2022truthfulqa}, IFEval prompt-level strict accuracy \citep{zhou2023instruction}, GSM8K \citep{cobbe2021training}, and MATH-Hard drawn from the MATH dataset \citep{hendrycks2021math}. ARC uses 25-shot normalized accuracy, and GSM8K uses 5-shot exact-answer evaluation. The per-task tables (\cref{tab:qwen3-results,tab:pythia-results}) preserve evaluator-reported standard errors, which do not measure training-seed variability. Mean accuracy is a point summary. Pythia's mean accuracy excludes GSM8K and MATH-Hard because all its variants score zero on both. The summary averages cited in \cref{sec:motivation} are: Low$-$All $+1.95/+0.51/+0.32$, Random$-$All $+1.58/+0.00/+0.06$, and Low$-$Random $+0.37/+0.51/+0.26$ points for Pythia-2.8B, Qwen3-4B, and Qwen3-8B, respectively. The High/High row and the cross-side rules (Low/High, High/Low) are shown for completeness only. They are not part of the matched Low--Random comparison and are not interpreted in the main text. The UltraFeedback and HelpSteer3 collections differ in recipe and evaluation protocol, so absolute scores are not comparable across them.

\begin{table*}[t]
\centering
\setlength{\tabcolsep}{3pt}
\begin{tabular*}{\textwidth}{@{\extracolsep{\fill}}lcccc@{}}
\toprule
\rowcolor{tableHead}
Benchmark & All/All & Random/Random & High/High & Low/Low \\
\midrule
\multicolumn{5}{@{}l}{\emph{Qwen3-4B}} \\
ARC & \metric{62.80}{1.41} & \bestmetric{62.97}{1.41} & \metric{62.46}{1.42} & \bestmetric{62.97}{1.41} \\
BBH & \metric{56.16}{3.01} & \metric{56.16}{3.01} & \metric{56.08}{3.01} & \bestmetric{56.33}{3.01} \\
MMLU-Pro & \metric{42.46}{0.45} & \metric{42.25}{0.45} & \metric{42.24}{0.45} & \bestmetric{42.50}{0.45} \\
MuSR & \metric{42.46}{1.73} & \bestmetric{43.39}{1.74} & \metric{42.72}{1.74} & \metric{43.25}{1.74} \\
TruthfulQA & \metric{37.94}{1.70} & \metric{38.07}{1.70} & \metric{36.11}{1.68} & \bestmetric{38.56}{1.70} \\
IFEval & \metric{43.99}{2.14} & \metric{40.67}{2.11} & \bestmetric{44.36}{2.14} & \metric{43.56}{2.10} \\
GSM8K & \metric{74.15}{1.21} & \metric{74.98}{1.19} & \metric{68.76}{1.28} & \bestmetric{75.36}{1.19} \\
MATH-Hard & \metric{37.94}{3.44} & \metric{39.39}{3.45} & \metric{37.25}{3.36} & \bestmetric{39.50}{3.46} \\
Mean accuracy (\%) & 49.74 & 49.74 & 48.75 & 50.25 \\
\midrule
\multicolumn{5}{@{}l}{\emph{Qwen3-8B}} \\
ARC & \metric{68.17}{1.36} & \metric{68.09}{1.36} & \bestmetric{68.34}{1.36} & \metric{68.29}{1.36} \\
BBH & \metric{60.84}{2.98} & \metric{60.80}{2.98} & \metric{60.55}{2.99} & \bestmetric{61.34}{2.98} \\
MMLU-Pro & \metric{47.63}{0.46} & \bestmetric{47.70}{0.46} & \metric{47.61}{0.46} & \metric{47.68}{0.46} \\
MuSR & \metric{43.39}{1.77} & \bestmetric{43.78}{1.77} & \metric{42.86}{1.77} & \metric{43.39}{1.77} \\
TruthfulQA & \metric{36.72}{1.69} & \bestmetric{36.96}{1.69} & \metric{36.84}{1.69} & \metric{36.84}{1.69} \\
IFEval & \metric{25.88}{1.88} & \metric{25.32}{1.87} & \metric{24.03}{1.84} & \bestmetric{26.06}{1.89} \\
GSM8K & \metric{78.09}{1.14} & \metric{77.63}{1.15} & \metric{77.86}{1.14} & \bestmetric{78.62}{1.13} \\
MATH-Hard & \metric{46.25}{3.45} & \metric{47.13}{3.44} & \metric{46.38}{3.49} & \bestmetric{47.29}{3.46} \\
Mean accuracy (\%) & 50.87 & 50.93 & 50.56 & 51.19 \\
\bottomrule
\end{tabular*}
\caption{UltraFeedback Qwen3 task scores (\%) with evaluator-reported standard errors. Mean accuracy is the eight-benchmark point mean. Non-All rules use retention 0.8. These errors do not measure training variability.}
\label{tab:qwen3-results}
\end{table*}

\begin{table*}[!t]
\centering
\small
\setlength{\tabcolsep}{3pt}
\begin{tabular}{lccccccc}
\toprule
\rowcolor{tableHead}
Method & ARC & BBH & MMLU-Pro & MuSR & IFEval & TruthfulQA & \shortstack{Mean\\accuracy (\%)} \\
\midrule
All/All & \metric{25.51}{1.27} & \metric{28.46}{2.64} & \metric{10.59}{0.28} & \metric{35.85}{1.71} & \metric{11.46}{1.37} & \metric{22.40}{1.46} & 22.38 \\
Random/Random & \bestmetric{28.16}{1.31} & \metric{27.47}{2.55} & \metric{10.86}{0.28} & \metric{36.64}{1.72} & \metric{15.53}{1.56} & \metric{25.09}{1.52} & 23.96 \\
Low/Low & \metric{27.73}{1.31} & \metric{28.47}{2.69} & \metric{10.84}{0.28} & \bestmetric{39.95}{1.73} & \metric{11.46}{1.37} & \bestmetric{27.54}{1.56} & 24.33 \\
Low/High & \metric{24.91}{1.26} & \metric{28.92}{2.69} & \bestmetric{11.66}{0.29} & \metric{35.58}{1.70} & \bestmetric{18.30}{1.66} & \metric{22.64}{1.47} & 23.67 \\
High/Low & \metric{26.11}{1.28} & \metric{29.47}{2.72} & \bestmetric{11.66}{0.29} & \metric{34.79}{1.69} & \metric{12.57}{1.43} & \metric{23.01}{1.47} & 22.93 \\
High/High & \metric{25.51}{1.27} & \bestmetric{29.99}{2.76} & \metric{11.16}{0.29} & \metric{34.39}{1.68} & \metric{12.20}{1.41} & \metric{23.13}{1.48} & 22.73 \\
\bottomrule
\end{tabular}
\caption{UltraFeedback Pythia-2.8B task scores (\%) at retention 0.8. Per-task errors are evaluator-reported standard errors. Mean accuracy is the six-task point mean; GSM8K and MATH-Hard are excluded because all variants score zero on both.}
\label{tab:pythia-results}
\end{table*}

\subsection{The HelpSteer3 Corpus and Recipe}
\label{app:hs-recipe}
\paragraph{Preference conversion and recipe.}
The preference conversion retains 36,264 of 38,459 training pairs and 1,917 of 2,017 validation pairs, removing ties and empty responses. The main experiments use Qwen3-4B with DPO $\beta=1$, LoRA rank 64 and alpha 128, dropout $0.05$ on attention and MLP projections, BFloat16 computation, FlashAttention-2, gradient checkpointing, AdamW at learning rate $10^{-6}$, cosine decay, warmup fraction $0.05$, and maximum gradient norm $0.5$. Sequence and prompt limits are 2,048 and 1,024. Per-device batch is 2. The factorial and parameter sweeps use training and selection seed 42. The reference is computed online. The supplemental Full--Branch comparisons retain this recipe at Qwen3-4B seed 43 and Mistral-7B seed 42 (\cref{app:supplement}).

\subsection{Retention-Sweep Geometry}
\label{app:hs-sweep}
The sweep in \cref{tab:hs-pairs} compares Low/Low and Random/Random at retentions $0.1$--$0.9$, with masks recomputed every step. Gradient accumulation is 16 for retentions $0.1$--$0.4$ and 4 for $0.5$--$0.9$, giving effective batches of 32 (1,133 updates) and 64 (566 updates) for one epoch. Each Low--Random pair within a block shares its geometry and data order. Comparisons are therefore interpreted within each retention level, and the two blocks do not form one controlled dose-response curve. \Cref{tab:hs-full} gives the full per-task scores.

\subsection{The Factorial and Parameter Sweeps}
\label{app:hs-factorial}
\paragraph{Factorial.}
The factorial of \cref{sec:setup} uses the same corpus and recipe with a fixed batch geometry: per-device batch 2, gradient accumulation 4, eight workers, global batch 64 pairs, 566 optimizer steps for one epoch. Position pruning uses $\alpha=0.6$ (uniform response-position subsets, selected-token mean reward). Sequence pruning uses $\rho_b\sim\mathcal U(0.5,1.0)$ (prefix truncation with repacking and a matching reference cutoff). Branch pruning uses $q_w=q_l=0.5$ (one backward branch per pair, opposite branch forwarded without gradient, $1/q_b$ importance correction and stop-gradient surrogate). Composition follows \cref{alg:stochastic-dpo}: truncation first, then response-position pruning inside the retained prefix, then branch pruning. Policy and reference see identical truncations and masks. All eight configurations in \cref{tab:factorial} were trained end-to-end under this recipe and evaluated on all ten benchmarks.

Realized sparsity matches the nominal targets. With all three modes enabled, the selected fraction is $0.6013$, the mean $\rho_b$ is $0.7545$, and the response retention is $0.7662$.

\begin{table*}[ht]
\centering
\footnotesize
\setlength{\tabcolsep}{1.2pt}
\begin{tabular*}{\textwidth}{@{\extracolsep{\fill}}l Z X YYYYYYY Y YYY G@{}}
\toprule
\rowcolor{tableHead}
& \multicolumn{2}{c}{\scriptsize Efficiency} & \multicolumn{8}{c}{\scriptsize Capability (accuracy, \%)} & \multicolumn{4}{c}{\scriptsize Open-ended}\\
\cmidrule(lr){2-3}\cmidrule(lr){4-11}\cmidrule(lr){12-15}
\rowcolor{tableHead}
{\scriptsize Method} & {\scriptsize Wall} & {\scriptsize Speed} & {\scriptsize ARC} & {\scriptsize GSM8K} & {\scriptsize TQA} & {\scriptsize BBH} & {\scriptsize MATH} & {\scriptsize MMLU} & {\scriptsize MuSR} & {\scriptsize\shortstack{Mean\\accuracy\\(\%)}} & {\scriptsize MT} & {\scriptsize AH} & {\scriptsize WB} & {\scriptsize\shortstack{Mean\\generation\\score}}\\
\midrule
\rowcolor{tableBase}
Full & 187.8 & 1.000$\times$ & \textbf{62.54} & \textbf{87.41} & \textbf{53.88} & 55.64 & 63.59 & \textbf{42.28} & 42.99 & \textbf{58.33} & 6.649 & \textbf{50.12} & 3.139 & \underline{49.33}\\
Position & 185.9 & 1.010$\times$ & \underline{62.46} & 86.43 & \underline{53.51} & 55.58 & 63.26 & 42.18 & 42.99 & 58.06 & \underline{6.829} & 44.85 & 3.200 & 48.38\\
Sequence & 186.4 & 1.008$\times$ & 62.37 & 86.28 & 53.40 & \textbf{55.92} & 63.27 & \underline{42.20} & 42.86 & 58.04 & 6.432 & 45.80 & 3.186 & 47.33\\
\rowcolor{tableFocus}
Branch & 128.4 & 1.463$\times$ & 62.03 & \underline{86.66} & 53.06 & 55.73 & \textbf{63.65} & 42.16 & \textbf{43.25} & \underline{58.08} & 6.800 & \underline{50.11} & \textbf{3.263} & \textbf{50.25}\\
Position+Sequence & 185.0 & 1.015$\times$ & 62.20 & 85.97 & 53.06 & \underline{55.74} & \underline{63.63} & 42.07 & \underline{43.12} & 57.97 & \textbf{6.893} & 45.54 & 3.208 & 48.85\\
Position+Branch & 127.8 & 1.469$\times$ & 61.95 & 86.50 & 52.92 & 55.64 & 63.31 & 42.12 & \underline{43.12} & 57.94 & 6.676 & 45.80 & 3.196 & 48.17\\
Sequence+Branch & \textbf{126.3} & \textbf{1.487$\times$} & 62.03 & 86.20 & 52.84 & 55.69 & 63.06 & 42.05 & \underline{43.12} & 57.86 & 6.600 & 45.05 & 3.197 & 47.67\\
\shortstack[l]{Position+Sequence\\+Branch} & \underline{127.0} & \underline{1.479$\times$} & 62.03 & 85.97 & 52.75 & 55.65 & 62.24 & 42.05 & \underline{43.12} & 57.69 & 6.693 & 45.86 & \underline{3.225} & 48.35\\
\bottomrule
\end{tabular*}
\caption{Full factorial task scores underlying \cref{tab:factorial}. Metrics and score units follow \cref{sec:setup}. Bold/underline mark the best/second-best displayed column values.}
\label{tab:factorial-tasks}
\end{table*}

\paragraph{Parameter sweeps.}
The within-axis sweeps of \cref{tab:param-sweep} are ten trainings that share the factorial's recipe and differ from it only in the ablated knob. The sweep midpoints ($\alpha{=}0.6$, $\mathcal U(0.5,1)$, $q_w{=}0.5$) coincide with the factorial's Position, Sequence, and Branch rows. \Cref{tab:b-profile} gives the sequence-pruning resource profile in full. Realized sparsity matches nominal for position pruning (selected fraction $0.302/0.601/0.902$) and branch pruning (global chosen fractions $0.298/0.498/0.699$ for $q_w=0.3/0.5/0.7$). For sequence pruning, the sampled cutoff matches its distribution (mean $\rho_b$ $0.632/0.755/0.877$ against theory $0.625/0.750/0.875$, and $\rho_b=0.750$ exactly for the fixed-cutoff configuration), with corresponding response-token retention $0.65/0.77/0.88$. The Tokens column in \cref{tab:b-profile} ($0.854/0.893/0.943/0.894$) is a third quantity: the executed-sequence fraction, the ratio of executed sequence tokens, prompt and padding included, to the stock baseline. It lies above the response-token retention and must not be compared with the nominal $\rho_b$. Peak HBM was not measured across $q_w$, so branch pruning is compared on training wall-clock only. The per-task scores of all ten runs are in \cref{tab:sweep-full}.

\begin{table*}[t]
\centering
\small
\setlength{\tabcolsep}{2pt}
\begin{tabular}{l VVVVVVV V VVV G}
\toprule
\rowcolor{tableHead}
{\footnotesize Setting} & {\footnotesize ARC} & {\footnotesize GSM8K} & {\footnotesize TQA} & {\footnotesize BBH} & {\footnotesize MATH} & {\footnotesize MMLU} & {\footnotesize MuSR} & {\scriptsize\shortstack{Mean\\accuracy\\(\%)}} & {\footnotesize MT} & {\footnotesize AH} & {\footnotesize WB} & {\scriptsize\shortstack{Mean\\generation\\score}}\\
\midrule
Position $\alpha{=}0.3$ & 62.29 & 86.20 & 53.15 & 55.61 & 62.78 & 42.15 & 43.12 & 57.90 & 6.568 & 45.56 & 3.195 & 47.73\\
\rowcolor{tableFocus}
Position $\alpha{=}0.6$ & 62.46 & 86.43 & 53.51 & 55.58 & 63.26 & 42.18 & 42.99 & 58.06 & 6.829 & 44.85 & 3.200 & 48.38\\
Position $\alpha{=}0.9$ & 62.54 & 87.34 & 53.83 & 55.64 & 63.54 & 42.26 & 43.12 & 58.32 & 6.707 & 48.00 & 3.081 & 48.63\\
Sequence $\mathcal U(0.25,1)$ & 62.37 & 86.35 & 53.30 & 55.80 & 63.53 & 42.15 & 43.12 & 58.09 & 6.552 & 47.05 & 3.117 & 47.91\\
\rowcolor{tableFocus}
Sequence $\mathcal U(0.5,1)$ & 62.37 & 86.28 & 53.40 & 55.92 & 63.27 & 42.20 & 42.86 & 58.04 & 6.432 & 45.80 & 3.186 & 47.33\\
Sequence $\mathcal U(0.75,1)$ & 62.46 & 86.28 & 53.44 & 56.03 & 63.16 & 42.16 & 42.72 & 58.04 & 6.874 & 47.92 & 3.193 & 49.53\\
Sequence fixed $0.75$ & 62.29 & 86.73 & 53.38 & 55.95 & 63.23 & 42.13 & 42.86 & 58.08 & 6.562 & 43.33 & 3.214 & 47.03\\
Branch $q_w{=}0.3$ & 62.37 & 87.34 & 53.03 & 55.44 & 62.70 & 42.16 & 43.39 & 58.06 & 6.669 & 47.82 & 3.189 & 48.80\\
\rowcolor{tableFocus}
Branch $q_w{=}0.5$ & 62.03 & 86.66 & 53.06 & 55.73 & 63.65 & 42.16 & 43.25 & 58.08 & 6.800 & 50.11 & 3.263 & 50.25\\
Branch $q_w{=}0.7$ & 62.03 & 86.58 & 52.97 & 55.75 & 62.92 & 42.15 & 43.39 & 57.97 & 6.753 & 46.64 & 3.205 & 48.74\\
\bottomrule
\end{tabular}
\caption{Per-task single-axis sweeps, one training seed. Shading marks defaults; rows specify settings. Metric units follow \cref{sec:setup}. Evaluator standard errors ($0.45$--$3.02$\,pp) do not test aggregate differences or training variability.}
\label{tab:sweep-full}
\end{table*}

Mean generation score in every table uses the reported task means and the scaling in \cref{sec:setup}.

\paragraph{Descriptive factorial contrasts.}
Raw differences of aggregate columns across the factorial are descriptive contrasts on single-seed runs. Contrasts use unrounded aggregates and can differ from differences of the displayed columns by $0.01$. Adding position pruning changes mean accuracy by $-0.27$, $-0.07$, $-0.14$, and $-0.17$\,pp over Full, Sequence, Branch, and Sequence+Branch, respectively. Adding sequence pruning changes it by $-0.29$, $-0.09$, $-0.22$, and $-0.25$\,pp over Full, Position, Branch, and Position+Branch. Adding branch pruning changes it by $-0.26$, $-0.12$, $-0.19$, and $-0.28$\,pp over Full, Position, Sequence, and Position+Sequence.

Contrasts in mean generation score use the scaling in \cref{sec:setup} and have mixed signs. Position pruning gives $-0.95$, $+1.52$, $-2.07$, and $+0.68$ in the four pairings above. Sequence pruning gives $-2.01$, $+0.47$, $-2.58$, and $+0.17$. Branch pruning gives $+0.92$, $-0.21$, $+0.35$, and $-0.50$. All twelve mean accuracy contrasts are negative. The branch-pruning effect on training wall-clock is consistent in sign: $-59.4$, $-58.1$, $-60.1$, and $-58.0$\,min in the four pairings.

\subsection{Full Task Tables}
\label{app:task-tables}
\Cref{tab:hs-full} reports per-task scores for Low/Low and Random/Random at each retention level on HelpSteer3-Preference.
\begin{table*}[t]
\centering
\small
\setlength{\tabcolsep}{3pt}
\begin{tabular}{llrrrrrrrr}
\toprule
\rowcolor{tableHead}
Selection & Retention & ARC & GSM8K & TruthfulQA & BBH & MATH-Hard & MMLU-Pro & MuSR & \shortstack{Mean\\accuracy (\%)}\\
\midrule
Low/Low & 0.1 & 63.48 & 85.82 & 54.57 & 55.54 & 62.39 & 42.26 & 42.86 & 58.13 \\
Random/Random & 0.1 & 62.88 & 86.73 & 53.11 & 55.41 & 63.36 & 42.00 & 42.59 & 58.01 \\
Low/Low & 0.2 & 63.82 & 86.58 & 55.19 & 55.32 & 60.29 & 42.33 & 42.59 & 58.02 \\
Random/Random & 0.2 & 63.05 & 86.35 & 53.62 & 55.06 & 63.14 & 42.02 & 42.46 & 57.96 \\
Low/Low & 0.3 & 63.82 & 87.26 & 55.49 & 55.47 & 58.47 & 42.35 & 42.59 & 57.92 \\
Random/Random & 0.3 & 63.23 & 86.50 & 54.46 & 55.51 & 63.20 & 42.04 & 42.72 & 58.24 \\
Low/Low & 0.4 & 63.99 & 88.02 & 55.70 & 55.43 & 57.67 & 42.32 & 42.86 & 58.00 \\
Random/Random & 0.4 & 63.40 & 86.88 & 54.85 & 55.64 & 61.92 & 42.10 & 42.59 & 58.20 \\
Low/Low & 0.5 & 63.14 & 87.19 & 53.81 & 55.63 & 63.66 & 42.25 & 43.12 & 58.40 \\
Random/Random & 0.5 & 63.05 & 86.50 & 53.38 & 55.63 & 63.92 & 42.22 & 43.12 & 58.26 \\
Low/Low & 0.6 & 63.40 & 87.26 & 53.85 & 55.60 & 63.33 & 42.22 & 43.12 & 58.40 \\
Random/Random & 0.6 & 63.05 & 86.81 & 53.52 & 55.71 & 64.11 & 42.20 & 42.99 & 58.34 \\
Low/Low & 0.7 & 63.31 & 87.26 & 53.88 & 55.53 & 63.39 & 42.25 & 42.99 & 58.37 \\
Random/Random & 0.7 & 63.23 & 86.81 & 53.67 & 55.77 & 64.17 & 42.20 & 42.86 & 58.39 \\
Low/Low & 0.8 & 63.14 & 87.34 & 53.91 & 55.60 & 63.28 & 42.27 & 42.99 & 58.36 \\
Random/Random & 0.8 & 63.05 & 87.49 & 53.76 & 55.78 & 64.01 & 42.24 & 42.59 & 58.42 \\
Low/Low & 0.9 & 63.23 & 87.26 & 53.92 & 55.57 & 63.43 & 42.30 & 42.99 & 57.55$^*$ \\
Random/Random & 0.9 & 63.14 & 87.11 & 53.86 & 55.61 & -- & 42.26 & 43.12 & 57.52$^*$ \\
\midrule
Qwen3-4B & -- & 62.80 & 85.90 & 52.27 & 55.59 & 62.80 & 41.98 & 43.39 & 57.82 \\
Unmasked DPO & -- & 62.63 & 87.57 & 53.91 & 55.61 & 63.33 & 42.30 & 43.12 & 58.35 \\
\bottomrule
\end{tabular}
\caption{HelpSteer3-Preference task scores (\%). Low and Random are compared within a retention level. Retentions 0.1--0.4 and 0.5--0.9 use different training schedules. Mean accuracy is the seven-task mean. $^*$At $0.9$ both means use six tasks, excluding MATH-Hard, which is missing for Random. The last two rows are baselines: the untrained Qwen3-4B starting checkpoint and unmasked DPO, which provides context but does not use a matched training schedule. A dash denotes a missing result.}
\label{tab:hs-full}
\end{table*}

\subsection{Detailed HelpSteer3 Tables}
\label{app:main-tables}
The following tables retain the count-matched motivation results and the per-axis parameter settings summarized in the main text; the completed factorial is in \cref{tab:factorial}. All factorial and sweep configurations use seed 42.

\begin{table}[t]
\centering
\small
\setlength{\tabcolsep}{4pt}
\begin{tabular}{rrrr}
\toprule
\rowcolor{tableHead}
$\alpha$ & Low/Low & Random/Random & Difference\\
\midrule
\multicolumn{4}{l}{Batch 32}\\
0.1 & 58.13 & 58.01 & +0.12\\
0.2 & 58.02 & 57.96 & +0.06\\
0.3 & 57.92 & 58.24 & $-0.32$\\
0.4 & 58.00 & 58.20 & $-0.20$\\
\midrule
\multicolumn{4}{l}{Batch 64}\\
0.5 & 58.40 & 58.26 & +0.14\\
0.6 & 58.40 & 58.34 & +0.06\\
0.7 & 58.37 & 58.39 & $-0.02$\\
0.8 & 58.36 & 58.42 & $-0.06$\\
0.9 & 57.55 & 57.52 & +0.03$^*$\\
\bottomrule
\end{tabular}
\caption{Qwen3-4B on HelpSteer3-Preference: mean capability scores (\%) for confidence-ranked (Low/Low) versus matched uniform (Random/Random) masks at retention $\alpha$. Comparisons are within rows. $^*$At $0.9$ both means use six tasks, excluding MATH-Hard, which is missing for Random. Differences use displayed means.}
\label{tab:hs-pairs}
\end{table}

\subsection{Supplemental Attribution and Full--Branch Validation}
\label{app:supplement}
The attribution and Full--Branch validation add four complete trainings and two resource profiles on the same eight-accelerator hardware. Results are in \cref{tab:a-attribution,tab:c-validation,tab:branch-profile}; this section gives reproduction and scoring details.

\paragraph{Output implementation control.}
One Qwen profiler compares Full-concat, selected-head execution at retention $1.0$, and selected-head execution at $0.6$. The configurations share weights, inputs, reference computation, precision, and the $2\times4\times8$ training geometry (per-device batch 2, gradient accumulation 4, eight workers). Four warm-up rounds precede 16 measured rounds. Configuration order rotates within rounds, memory counters reset before each measurement, and peak HBM is the maximum over ranks. Full-retention selected-head rewards match Full-concat exactly, as do dense-reduced and selected-head rewards under identical masks. Numerical checks also compare LoRA gradients with a row-order control. Mean step times and peak HBM are reported in \cref{tab:a-attribution}. Memory differences are computed within this profile: $8{,}905$\,MB for full-supervision selected-head execution and a further $642$\,MB for random deletion.

\paragraph{Matched training pairs.}
Qwen Full and Branch at seed 43 use the same initialization, dataset version, data order, training budget, optimizer, schedule, precision, attention backend, checkpointing, and online reference. Mistral-7B-Instruct-v0.3 uses the same recipe and 566-step budget at seed 42. Both models fit the $2\times4\times8$ geometry without a batch-size change. Branch pruning uses $q_w=q_l=0.5$, with position and sequence pruning disabled. The global chosen-branch fractions are $0.5088$ for Qwen seed 43 and $0.4956$ for Mistral, averaged over all optimizer steps.

The Mistral setup is verified on chat template, prompt masking, EOS padding, the online reference, and the selected/no-gradient branches. The gradient checks in \cref{sec:c-validation} disable LoRA dropout while retaining checkpointing; the trainings use dropout $0.05$.

\paragraph{Per-task results.}
\Cref{tab:supplement-tasks} retains every benchmark score for the repeated Qwen pair and the Mistral pair. \Cref{tab:supplement-tasks} retains the native open-ended scores; \cref{tab:c-validation} reports category means. The conventional Branch-minus-Full differences are $-0.25$ and $-0.45$\,pp for the two Qwen seeds (a descriptive mean of $-0.35$\,pp), and $-0.31$\,pp for Mistral.

\begin{table*}[t]
\centering
\scriptsize
\setlength{\tabcolsep}{3pt}
\begin{tabular*}{\textwidth}{@{\extracolsep{\fill}}ll ZZZZZZZ Z ZZZ G@{}}
\toprule
\rowcolor{tableHead}
Model / seed & Config. & ARC & GSM8K & TQA & BBH & MATH & MMLU & MuSR & \shortstack{Mean\\accuracy\\(\%)} & MT & AH & WB & \shortstack{Mean\\generation\\score}\\
\midrule
Qwen3-4B / 42 & Full & 62.54 & 87.41 & 53.88 & 55.64 & 63.59 & 42.28 & 42.99 & 58.33 & 6.649 & 50.12 & 3.139 & 49.33\\
\rowcolor{tableFocus}
 & Branch & 62.03 & 86.66 & 53.06 & 55.73 & 63.65 & 42.16 & 43.25 & 58.08 & 6.800 & 50.11 & 3.263 & 50.25\\
\midrule
Qwen3-4B / 43 & Full & 62.54 & 87.04 & 53.88 & 55.55 & 63.45 & 42.34 & 42.99 & 58.26 & 6.623 & 44.36 & 3.233 & 47.64\\
\rowcolor{tableFocus}
 & Branch & 62.03 & 86.20 & 53.05 & 55.55 & 62.57 & 42.14 & 43.12 & 57.81 & 6.750 & 45.03 & 3.295 & 48.49\\
\midrule
Mistral-7B / 42 & Full & 67.06 & 52.46 & 64.82 & 46.37 & 12.50 & 30.47 & 40.34 & 44.86 & 5.130 & 24.60 & 2.361 & 33.17\\
\rowcolor{tableFocus}
 & Branch & 65.27 & 51.18 & 62.87 & 46.51 & 12.96 & 30.32 & 42.72 & 44.55 & 5.373 & 24.00 & 2.405 & 33.93\\
\bottomrule
\end{tabular*}
\caption{Complete Full--Branch scores in the units of \cref{sec:setup}. Qwen seed-42 values repeat the factorial evaluations in \cref{tab:factorial}.}
\label{tab:supplement-tasks}
\end{table*}

\paragraph{Open-ended scoring.}
All supplemental generations are judged by MiniMax-M3. The evaluator reports valid-item means separately for each checkpoint. The supplemental WildBench suite has 260 prompts, four more than the main evaluation's 256. The Qwen seed-43 Full Arena-Hard generation has 499 judged entries and 443 valid scores. \Cref{tab:supplement-tasks} reports per-checkpoint scores. These summaries do not characterize training-seed variability.

\section{Additional Token-Selection Analysis}
\label{app:moved}
This section reports the matched token-selection comparisons referenced in the main text.

\subsection{Supervision Membership Is a Limited Lever}
\label{sec:motivation}
We compare selecting positions with the lowest current-policy target-token probabilities (Low) against uniform random selection (Random). Low/Low and Random/Random use the same selection rule on both chosen and rejected responses. At equal retained ratios, Low/Low exceeds Random/Random by $0.37$, $0.51$, and $0.26$ mean-accuracy points on UltraFeedback \citep{ultrafeedback} for Pythia-2.8B \citep{pythia}, Qwen3-4B, and Qwen3-8B. On HelpSteer3-Preference, the difference changes sign across retention levels and stays within $0.32$ points (\cref{tab:hs-pairs}). Task differences are larger: at low retention, Low trails Random on MATH-Hard while leading TruthfulQA (\cref{tab:hs-full}). These comparisons motivate uniform position pruning before vocabulary projection. We compare selection rules at a fixed number of scored positions. The factorial experiment separately assesses the model-quality cost of reduced supervision relative to the Full baseline.

\section{LLM Usage Disclosure}
\label{app:llm_disclosure}
In answer to the submission questions on the use of large language models, a large language model was used to aid and polish the writing of this paper, for retrieval and discovery such as finding related work, and to draft and revise sections of the manuscript.

\end{document}

%% file: math_commands.tex
\usepackage{amsmath,amsfonts,bm}

\def\eqref#1{equation~\ref{#1}}

\def\ceil#1{\lceil #1 \rceil}

\def\1{\bm{1}}
\newcommand{\train}{\mathcal{D}}

\def\vtheta{{\bm{\theta}}}
\def\vphi{{\bm{\phi}}}

\def\vc{{\bm{c}}}

\def\vg{{\bm{g}}}

\def\vw{{\bm{w}}}
\def\vx{{\bm{x}}}
\def\vy{{\bm{y}}}
\def\vz{{\bm{z}}}

\def\evy{{y}}

\def\mSigma{{\bm{\Sigma}}}

\DeclareMathAlphabet{\mathsfit}{\encodingdefault}{\sfdefault}{m}{sl}
\SetMathAlphabet{\mathsfit}{bold}{\encodingdefault}{\sfdefault}{bx}{n}

\def\sB{{\mathbb{B}}}

\def\sI{{\mathbb{I}}}

\def\sS{{\mathbb{S}}}

\def\sU{{\mathbb{U}}}

\newcommand{\E}{\mathbb{E}}
\newcommand{\Ls}{\mathcal{L}}

\newcommand{\sigmoid}{\sigma}
\newcommand{\softplus}{\zeta}

\newcommand{\Var}{\mathrm{Var}}

\newcommand{\Cov}{\mathrm{Cov}}

\DeclareMathOperator{\prefixsample}{Prefix}
\DeclareMathOperator{\positionsample}{Select}
\DeclareMathOperator{\branchsample}{Branch}
\DeclareMathOperator{\logprobs}{LogProbs}
\DeclareMathOperator{\mean}{mean}
\DeclareMathOperator{\sg}{sg}